\documentclass{article}

\PassOptionsToPackage{numbers,compress}{natbib}
\usepackage{natbib}
\usepackage[final]{neurips_2025}
\usepackage[ruled,vlined]{algorithm2e}
\usepackage{amsmath}
\usepackage{todonotes}
\usepackage{newtxtext}
\usepackage{wrapfig}
\usepackage{bold-extra}

\usepackage[utf8]{inputenc} 
\usepackage[T1]{fontenc}    
\usepackage{hyperref}       
\usepackage{url}            
\usepackage{booktabs}       
\usepackage{amsfonts}       
\usepackage{amsthm}
\usepackage{amsmath}
\usepackage{nicefrac}       
\usepackage{microtype}      
\usepackage{xcolor}         
\usepackage{dsfont}

\newtheorem{theorem}{Theorem}[section]

\newtheorem{proposition}[theorem]{Proposition}
\newcommand{\scout}{\textsc{scout}\xspace}
\newcommand{\test}{{\sc test}\xspace}
\newcommand{\testbrute}{{\sc test-bruteforce}\xspace}
\newcommand{\testbi}{{\sc test-bisection}\xspace}
\newcommand{\alphabeta}{{\sc alpha-beta}\xspace}
\newcommand{\solve}{\textsc{solve}\xspace}
\newcommand{\ab}{{\sc ab}\xspace}

\author{Raphaël Boige\thanks{Corresponding author: \texttt{\{name\}.\{surname\}@inria.fr}} \qquad Amine Boumaza \qquad Bruno Scherrer\\Université de Lorraine, CNRS, Inria, LORIA, F-54000 Nancy, France}

\usepackage{setspace}
\AtBeginDocument{%
  \addtolength\abovedisplayskip{-0.4\baselineskip}%
  \addtolength\belowdisplayskip{-0.4\baselineskip}%
\addtolength\abovedisplayshortskip{-0.6\baselineskip}%
 \addtolength\belowdisplayshortskip{-0.6\baselineskip}%
}

\title{AlphaBeta is not as good as you think: a simple class of synthetic games for a better analysis of deterministic game-solving algorithms}
\date{February 2025}

\begin{document}

\maketitle
\begin{abstract}
Deterministic game-solving algorithms are conventionally analyzed in the light of their average-case complexity against a distribution of random game-trees, where leaf values are independently sampled from a fixed distribution. This simplified model enables uncluttered mathematical analysis, revealing two key properties: root value distributions asymptotically collapse to a single fixed value for finite-valued trees, and all reasonable algorithms achieve global optimality. However, these findings are artifacts of the model’s design: its long criticized independence assumption strips games of structural complexity, producing trivial instances where no algorithm faces meaningful challenges. To address this limitation, we introduce a class of synthetic games generated by a probabilistic model that incrementally constructs game-trees using a fixed level-wise conditional distribution. By enforcing ancestor dependencies, a critical structural feature of real-world games, our framework generates problems with adjustable difficulty while retaining some form of analytical tractability. For several algorithms, including AlphaBeta and Scout, we derive recursive formulas characterizing their average-case complexities under this model. These allow us to rigorously compare algorithms on deep game-trees, where Monte-Carlo simulations are no longer feasible. While asymptotically, all algorithms seem to converge to identical branching factor (a result analogous to that of independence-based models), deep finite trees reveal stark differences: AlphaBeta incurs a significantly larger constant multiplicative factor compared to algorithms like Scout, leading to a substantial practical slowdown. Our framework sheds new light on classical game-solving algorithms, offering rigorous evidence and analytical tools to advance the understanding of these methods under a richer, more challenging, and yet tractable model.

\end{abstract}

\section{Introduction}  
\label{sec:introduction}
\setcounter{footnote}{0} 
In this work, we consider a class of deterministic two-player zero-sum games represented by trees of \emph{height} \( h \), where each node has a uniform \emph{branching degree} \( b \), as illustrated in Figure~\ref{fig:tree}. Each level alternates between decision points for the maximizing and minimizing players (with the root always being a max node). Internal nodes propagate values from their children via alternating $\min$/$\max$ operators, reflecting optimal play. Clearly, the entire tree is determined by its leaf values, and solving it involves recursively applying $\min$ and $\max$ operations until the root value is resolved.

Game-solving algorithms are conventionally evaluated~\cite{fuller1973analysis, baudet1978branching, pearl1982solution} by the number of leaf evaluations required. For example, brute-force search evaluates all \( b^h \) leaves, corresponding to a \textit{branching factor} of \( b \)---the average nodes evaluated per level, see Section~\ref{sec:preliminaries} for formal definitions. Notably, even with prior knowledge of the root value, verifying it requires evaluating at least one node per max level and all \( b \) nodes per min level. This results in a complexity of \( \smash{b^{\frac{h}{2}}} \), or a branching factor of \( \smash{\sqrt{b}} \), establishing
\begin{figure}[t!]
\centering
\includegraphics[width=\textwidth]{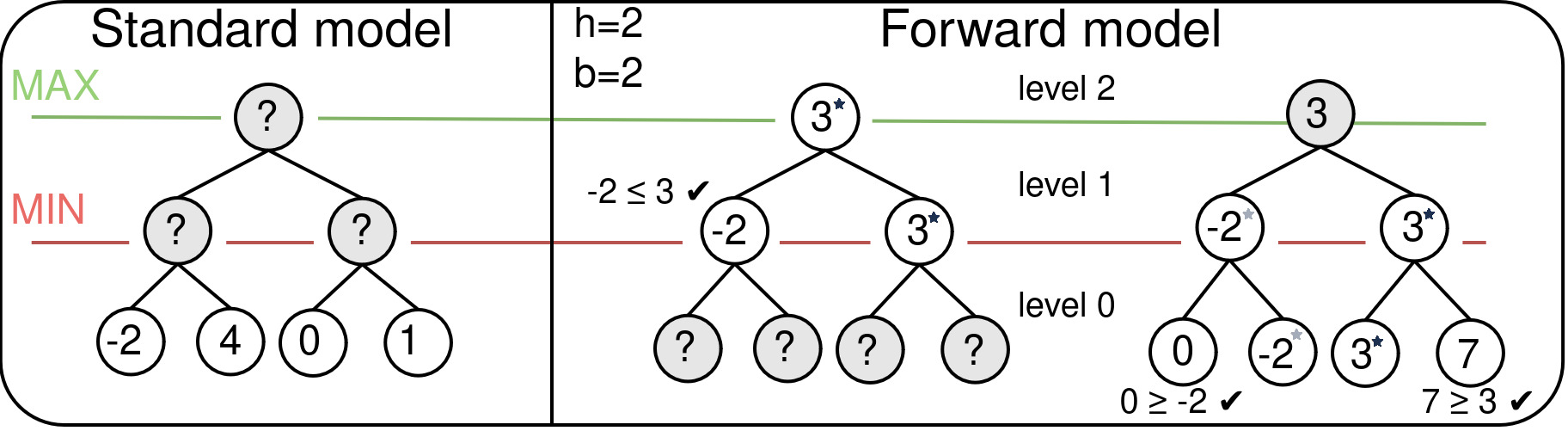}
\label{fig:tree}
\vspace{-0.6cm}
\caption{Illustration of a game-tree of height $h=2$ and branching degree $b=2$. (Left) In the \textit{standard model}, leaf values are independently sampled from a distribution. (Right) In the \textit{forward model}, intermediate values are sampled progressively, level by level, until leaf nodes are reached. One random child is chosen to inherit the root value, illustrated by a star symbol, and the remaining ones are sampled according to a fixed distribution, truncated to respect the minimax constraints.\vspace{-0.4cm}}
\end{figure}
global upper and lower bounds for all algorithms. \alphabeta, as shown in~\cite{knuth1975analysis}, achieves these bounds under optimally ordered worst-case and best-case trees. This has motivated the emergence of average-case analyses, which aim to understand how game-solving algorithms perform on a diverse collection of random trees. The classical approach, hereafter called the \textit{standard model}\footnote{The \textit{standard model} is often referred to as the Pearl-game or P-game model in the literature.}, samples leaf values independently from a fixed distribution. While mathematically tractable, this model exhibits critical flaws: when leaf values are restricted to a finite set, Pearl~\cite{pearl1980asymptotic} proved that the root value asymptotically collapses to a single fixed value for all distributions. This collapse renders algorithm comparisons questionable: instances become trivial, as methods merely confirm a predetermined value shared among trees.

A practical consequence is the \textit{standard model}'s assignment of global optimality to multiple algorithms. \alphabeta, for instance, provably achieves the \( \smash{\sqrt{b}} \) branching factor asymptotically when $h$ and $b$ tend to infinity, except for one rare case discussed later. This is remarkable, since no alternative algorithm, even in principle, can outperform them asymptotically. While this has been interpreted as a sign of algorithmic maturity, we believe instead that it reflects deficiencies in the evaluation framework. When trees homogenize to trivial instances with a single root value, comparisons lose meaning as algorithms face no substantive challenges to distinguish their performance. Previous works have pinpointed that these limitations stem from the model's independence assumption~\cite{nau1982investigation,nau1982lastplayer, knuth1975analysis}, which can't model dependencies between sibling nodes. Real-world games like chess possess an important type of ancestor dependencies: correlation among sibling nodes. In average a winning position for the white player should have many wins (and few losses) in the subtree emanating from it. Crucially, the independence assumption eliminates such complexity, rendering the \textit{standard model} a poor proxy for practical scenarios, and a poor benchmark for comparing algorithms.

In this work, we specify and analyze a synthetic game-tree model: the \textit{forward model}\footnote{In the rest of this work, we freely use the terms \textit{Forward-game} or \textit{F-game} to designate a game generated by the \textit{forward model}.}, that addresses these limitations. By constructing trees level-by-level with a conditional distribution that enforces ancestor dependencies, our approach captures two critical properties: (1) sibling values correlate conditioned on their parent value, coarsely mimicking the strategic continuity of real games, and (2) game difficulty can be better modulated, enabling more challenging benchmarks for algorithm analysis. In Section~\ref{sec:forward} we further detail our \textit{forward model}. Within this framework, we characterize the behavior of classic algorithms with recursive formulas for their average-case complexity, including \alphabeta, which we develop in Section~\ref{sec:alphabeta_complexit}. These formulas are convenient for asymptotic analysis as well as limit-depth evaluations, as they allow simulating the behavior of algorithms on deep trees much more efficiently than Monte-Carlo simulations. Our findings reveal that in the asymptotic regime, all algorithms share the same branching factor, as a function of the distribution chosen. Theoretically, this suggests that there is little reason to prefer one algorithm over another. However, finite-depth analysis, in Section~\ref{sec:exp}, reveals critical practical differences masked by asymptotics. Specifically, \alphabeta incurs a larger multiplicative constant, and needs in average to evaluated more leaf nodes compared to other algorithms like \scout. Finally, to allow reproduction of the numerical results presented in this paper, we open-source our codebase which permits the precise computation of average-case complexities of several algorithms and for game-trees up to height $h\approx 5000$\footnote{Link to the repository: \href{https://github.com/Egiob/alphabeta}{https://github.com/Egiob/alphabeta}}.

\section{Notations and background}
\label{sec:preliminaries}
In this section, we formalize the framework for analyzing deterministic game-solving algorithms, focusing on average-case complexity under a probabilistic tree generation model.
\paragraph{Game formalism} We model games as complete $b$-ary trees of \emph{height} $h$, where leaf nodes hold values from a space $\mathcal{V}$. Formally, a minimax node value $V_h$ with children values $\smash{(V^i_{h-1})_{i\leq b}}$ can be written $V_h = \max_{i\leq b}(V_{h-1}^i)$ (\textit{resp.} $\min_{i\leq b}(V_{h-1}^i)$) if level $h$ is at an even (\textit{resp.} odd) distance from the root.
However, in two-player zero-sum games, by using the identity $\min(a,b) = -\max(-a,-b)$, we can formulate an equivalent negamax view where the alternate min/max are replaced by a single operator. This defines a negamax node value $W_h$ with children values $\smash{(W^i_{h-1})_{i\leq b}}$ as:
\begin{equation}
    \label{eq:negamax}
    W_h = \max_{i\leq b}(-W_{h-1}^i)
\end{equation}
We can see the two views are equivalent,     as $W_h = V_h$ (\textit{resp.} $-V_h$) if level $h$ is at an even (\textit{resp.} odd) distance from the root. In the rest of this work, we use the negamax view, as it often simplifies algorithmic descriptions and formal analysis, conveniently halving the number of cases to study.

\paragraph{Algorithmic complexity and branching factor} To measure the efficiency of a deterministic algorithm $A$, the main value of interest is its \emph{average-case complexity} \( I_A(h) \), \textit{i.e.} the expected number---over a distribution of trees generated randomly---of \emph{leaf node} inspections required by the algorithm $A$ to terminate.

The \emph{asymptotic branching factor} \( r_A \) quantifies the complexity growth of the algorithm with the height of the tree:
\begin{equation}
    r_A = \lim_{h \to \infty} \sqrt[h]{I_A(h)}.
\end{equation}

This represents the effective branching rate per level. As stated before, global lower and upper bounds are known for these quantities, for all algorithms $A$: \(b^{\frac h 2} \leq I_ A(h) \leq b^h \text{ and } \sqrt{b} \leq r_A \leq b\).

\section{The forward model}
\label{sec:forward}
\begin{wrapfigure}[14]{r}{0.47\textwidth} 
\vspace{-0.5cm}
\begin{algorithm}[H]
\SetAlgoLined
\DontPrintSemicolon
\KwIn{Node value $x$, tree height~$h$, branching degree $b$.}
\KwOut{Children values list.}
\lIf{$h = 0$}{\Return{$[\ ]$} \tcp*[f]{Leaf nodes}} 
$x'_{list} \leftarrow[\ ]$ \\
$k_{sample} \sim \mathcal{U}\{1,b\}$\\
\For{k=1...b}{
    \lIf{$k= k_{sample}$}{$x' = -x$}
    \lElse{$x' \sim  \mu(\cdot \mid \cdot \geq -x)$}
    $x'_{list}.append(x')$ 
}
\Return{$x'_{list}$}
\caption{FORWARD-SAMPLE($x$,$h$,$b$)}
\label{alg:forward_sample}
\end{algorithm}
\end{wrapfigure}

This section formally introduces the \textit{forward model}: a simple synthetic game model that recursively generates game-trees, down from the root node, through a level-wise sampling process that enforces the negamax constraint from Equation~\ref{eq:negamax} at each step. This model resembles previous synthetic game models that generate trees in a top-down fashion~\cite{scheucher1998benefits,furtak2009minimum}. Still, to the best of our knowledge, this work provides the first analysis of this conceptually simple yet rich game model. The sampling process can be described as follows: starting from the root node, assuming it holds value $x$, one of its $b$ children is uniformly selected to inherit the parent’s negated value $-x$, ensuring compliance with the negamax constraint. The remaining $(b-1)$ children are then sampled from the level-wise distribution $\mu$, with support dynamically truncated  conditionally on $x$. For instance, if $\mu$'s initial support is $\{-n,...,n\}$ for an integer $n$, it gets truncated (and re-normalized) to $\{-x,..., n\}$. This procedure, dubbed \textsc{forward-sample} and formally described in Algorithm~\ref{alg:forward_sample}, is first called on the root node, and recursively applied to each node until all leaves have been generated. For simplicity, we choose the same distribution $\mu$ for each node, and we draw the root value of the tree according to $\mu$ as well.

Unlike traditional models that first assign leaf values and propagate them backward through min/max rules, our approach builds values progressively from the root downward, hence the "forward" designation. As demonstrated in the next section, this property simplifies the complexity analysis for game-solving algorithms, as intermediate outcomes are conditionally known during tree construction.

\section{A binary-valued example: the analysis of the SOLVE algorithm}

\label{sec:binarygames}
In this section, we consider binary-valued trees that represent two-outcome games, loss or win (i.e., $\mathcal{V}=\{0,1\}$). We focus on \solve, a canonical algorithm for two-outcome games, and contrast its behavior under the \textit{standard model} versus our \textit{forward model}. Key derivations appear in Appendix~\ref{app:solve}.

\paragraph{Algorithm description and standard model limitations}

The \solve algorithm (pseudo-code in Appendix~\ref{app:solve}) determines win/loss (1/0) outcomes by iteratively scanning through children and evaluating them until an opponent loss (0) is found, in this case it early stops and returns a win (1), if none is found, it returns a loss (0).  Under the \textit{standard model}, with probability $q_0$ of drawing a $0$ for a leaf node, \solve almost always achieves a globally optimal branching factor:
\begin{equation}
\label{eq:solve_std}
    r_{\solve}^{\textsc{standard}} =
    \begin{cases} \sqrt{b}&  \text{if $q_0\neq 1-\xi_b$ \ (very easy instances)}\\
    \frac {\xi_b}{(1-\xi_b)}=\mathcal{O}(\frac b {\log{b}}) & \text{if $q_0=1-\xi_b$ \ (special case)}
    \end{cases},
\end{equation}
where $\xi_b$ is the positive root of $x^b+x-1=0$~\cite{pearl1980asymptotic}. Since only this exceptional regime generates non-trivial instances, it has been thoroughly analyzed in the literature \cite{baudet1978branching, pearl1980asymptotic}. However, even this hardest regime generates much easier instances than those of an ordering-invariant worst-case model~\cite{saks1986probabilistic}:
\begin{equation}
\label{eq:saks}
   r_{\textsc{solve}}^{\textsc{rand-wc}} = \frac{b - 1 + \sqrt{b^2 + 14b + 1}}{4} = \mathcal{O}(b).
\end{equation}
The \textit{standard model}’s abrupt transition between $\smash{\sqrt{b}}$ and $\log b /b$ regimes, described in Equation~\ref{eq:solve_std}, reveals its inability to generate smoothly tunable or maximally hard instances.

\paragraph{Complexity analysis}
Under the \textit{forward model}, where $\mu = \mathcal{B}(q)$ (Bernoulli distribution with $q$ the probability of drawing a $0$), we derive (Appendix~\ref{app:solve}) a closed-form expression for \solve's branching factor:  
\begin{equation} 
\label{eq:solve_r_t}
    r_{\text{\solve}} = \frac{ t(q,b) + \sqrt{t(q,b)^2 + 4b} }{2}, \quad t(q,b)= \sum_{k=1}^{b-1} \tfrac{1 + (b - k - 1)q}{b} k (1 - q)^k.
\end{equation}  

If $q=1$, the model collapses to a trivial tree (alternating levels full of zeros and full of ones) leading to a $\smash{\sqrt{b}}$ branching factor. However, if $q=0$, it matches exactly the worst-case complexity of Equation~\ref{eq:saks}. For values of $q$ in-between, the monotonicity of $r_{\text{\solve}}$ with respect to $q$ guarantees that any branching factor from easiest to hardest case can be reached, allowing adjustable difficulty calibration (unattainable under the \textit{standard model}). In particular, a continuous range of $q$ values leads to asymptotically hardest instances as stated in the following theorem, the proof of which is in Appendix~\ref{app:solve_proof}.

\newcommand{\thmtexta}{For $b\in\mathbb N$ and for all $q \in [0,\frac 1 b]$, the branching factor of \solve satisfies $r_{\text{\solve}} = \mathcal{O}(b)$.}

\newtheorem*{Ta}{Theorem~\ref{th:solve1b}}
 
\begin{theorem}
\label{th:solve1b}
\thmtexta
\end{theorem}
In the next section, we analyze classical algorithms on a more general type of trees under our original \textit{forward} tree model.

\vspace{-3mm}
\section{Average-case analysis of classic algorithms}  
\label{sec:alphabeta_complexit}  

    Even though binary-valued games offer a simplified analysis, they cannot reflect the diversity of real games, which are best modeled using a broader value range. In this section, we analyze the algorithms \test, \alphabeta and \scout, on trees with values in $\{-n,\ldots, n\}$, \textit{i.e.} $\mu$ is a categorical distribution $\mathit{Cat}(p_{-n},\ldots,p_n)$.

\subsection{Analysis of TEST}  
\label{subsec:test_complexity} 

\begin{wrapfigure}[10]{r}{0.48\textwidth} 
\vspace{-2.1cm}
\begin{algorithm}[H]
\SetAlgoLined
\DontPrintSemicolon
\KwIn{Current node \textit{N}, search depth $h$, threshold $s$}
\KwOut{Certificate value determining if \textit{N}'s value is greater or equal than $s$}

\lIf{$h = 0$}{\Return{N.value}}

$best \leftarrow -\infty$\;
\ForEach{N' in N.children}{
    $value \leftarrow -\text{TEST}(N',\, h-1,\, -s+1)$
    
    $best \leftarrow \max(best, value)$ 
    
    \lIf{$best \geq s$}{\textbf{break} }
}   
\Return{$best$}\;
\caption{TEST($N,h,s$), \textit{negamax} form}
\label{alg:negamax-test}

\end{algorithm}
\end{wrapfigure}

\paragraph{Algorithm description} Described in Algorithm~\ref{alg:negamax-test}, \test answers, given a threshold $s$, whether the root value $x$ satisfies $x$$\geq$$s$. Like \solve, \test iteratively evaluates every child node by calling a negated version of itself and terminates early whenever it finds a value validating the condition $x$$\geq$$s$. It is almost identical to applying \solve to a binarized version of the same tree where leaves $l_i$ are converted to 1 if $l_i$$\geq$$s$ and 0 otherwise---with the main difference being \test not only returns the binary result of the assertion but also returns a certificate value determining whether $x$$\geq$$s$ or not. This makes it a useful building block for other game-solving algorithms. For instance, a simple algorithm  could brute-force over all possible thresholds $s$ to identify $x$, note that this approach can be optimized via bisection. In Section~\ref{sec:exp}, we compare these \testbrute and \testbi approaches to the \alphabeta and \scout algorithms. In the \textit{standard model}, if leaf values are drawn according to a distribution $\nu_0$ with cumulative distribution $F_{\nu_0}$, then:
\begin{equation}
    r_{\textsc{test}}^{\textsc{standard}}(\nu_0, s) = r_{\textsc{solve}}^{\textsc{standard}}(q_0=F_{\nu_0}(s)).
\end{equation}
This equivalence suggests that games with discrete values are not fundamentally asymptotically harder than games with binary values. In the following we characterize \test under the \textit{forward model} and show that its branching factor coincides with that of \solve on a worst-case distribution.
\vspace{-3mm}\paragraph{Complexity analysis} For a given threshold value \(s \in \{-n+1,\ldots, n\}\) we are interested in the average-complexity of \textsc{test($s$)} defined as \(I_{\text{\test}}^{s}(h) = \mathbb{E}_{X\sim\mu}[I_{\text{\test}}^{X, s}(h)\)], where \(I_{\text{\test}}^{x, s}(h)\) denotes the expected complexity of \test when the root value is \(x \in \{-n,\ldots,  n\}\). To model intermediate evaluation states, we extend this definition to \(I_{\text{\test}}^{x, s}(h, c)\), representing the complexity when the current node (with value $x$) has \(c \leq b\) remaining children to evaluate (all other nodes still have \(b\) children). Thus, the base case satisfies:  $I_{\text{\test}}^{x, s}(h) = I_{\text{\test}}^{x, s}(h, c = b)$. A necessary tool for expressing $I_{\text{\test}}^{x, s}(h)$ is the auxiliary function $J_{\text{\test}}^{x, s}(h, c)$, which represents the same complexity value, but conditioned on the probabilistic event that the "special child" (inheriting the root's value, enforced by the negamax constraint) has already been identified. A recursive system characterizing the average-case dynamics of \test can be derived by analyzing its execution flow under the \textit{forward model}. Upon evaluating the first child ($c=b$): (1) with probability \(1/c\), \test encounters the "special child", inheriting the root value \(-x\). The algorithm must then fully evaluate this child node by a recursive call to a negated version of \test at height $h-1$, if no cutoff occurs (\textit{i.e.} \(x < s\)) it continues evaluating the root node (height $h$), but with only \(c - 1\) remaining children, and knowing the "special child" has been found (cost of $J_{\text{\test}}^{x, s}(h, c-1)$ instead of $I_{\text{\test}}^{x, s}(h, c-1)$); (2) with probability \((c - 1)/c\), \test encounters a "normal" child whose value \(X'\) is sampled from \(\mu\) truncated to \(\{-x,\ldots, n\}\) and in the absence of cutoff (\textit{i.e.} \(-X' < s\)), the algorithm proceeds to evaluate the remaining children of the root node (cost of $I_{\text{\test}}^{x, s}(h, c-1)$). This gives the following equation for $I_{\text{\test}}^{x, s}$: 
\begin{equation}
\label{eq:i_test_recurrence}
\begin{split}  
I_{\text{\test}}^{x, s}(h, c) = \frac{1}{c} \biggl[\overbrace{ I_{\text{\test}}^{-x, -s+1}(h-1, b) + \mathds{1}_{\{x < s\}} J_{\text{\test}}^{x, s}(h, c-1)}^{\text{Special child $X'=-x$}} \biggr] \\  
+ \frac{c - 1}{c} \underset{\substack{X' \sim \mu \\ X' \geq -x }}{\mathbb{E}} \biggl[\underbrace{ I_{\text{\test}}^{X', -s+1}(h-1, b)  + \mathds{1}_{\{-X' < s\}} I_{\text{\test}}^{x, s}(h, c-1) }_{\text{Normal child $X'\sim \mu$}}\biggr].
\end{split}  
\end{equation}

The auxiliary function \(J_{\text{\test}}^{x, s}(h, c)\) follows the same logic, but being conditioned on the event "the special child already has been found", it only allows one of the branches, so the equation simplifies to:
\begin{equation}  
\label{eq:j_test_recurrence}
J_{\text{\test}}^{x, s}(h, c) = \underset{\substack{X' \sim \mu \\ X' \geq -x }}{\mathbb{E}} \biggl[ I_{\text{\test}}^{X', -s+1}(h-1, b) + \mathds{1}_{\{-X' < s\}} J_{\text{\test}}^{x, s}(h, c-1) \biggr].
\end{equation}

Note that for the end case $h=0$, both $I$ and $J$ equal $1$ (a tree with only one node incurs a cost of $1$), and for $c=0$, both $I$ and $J$ equal 0 (because no remaining child incurs no additional costs). Equations~\ref{eq:i_test_recurrence}~and~\ref{eq:j_test_recurrence}, characterizing the complexity of \test ---as well as Equations~\ref{equation:ab_complexity}~and~\ref{equation:ab_complexity_j} (resp. Equations \ref{eq:i_scout}~and~\ref{eq:j_scout}) characterizing the complexity of \alphabeta (resp. \scout)---were numerically validated through an extensive comparison with Monte-Carlo simulations, see Appendix~\ref{app:montecarlo}.

Conveniently, the intrinsic linear nature of Equations~\ref{eq:i_test_recurrence}~and~\ref{eq:j_test_recurrence} makes it possible to write the system in matrix form. This facilitates the efficient numerical computation  of the complexity of \textsc{test($s$)} by matrix iteration and that of its branching factor as the spectral radius (eigenvalue with highest magnitude) of this matrix. Additionally, we define a global branching factor for the \test algorithm, corresponding to the complexity of the average \test, or equivalently, the complexity of the hardest \test (over all threshold values $s$): $r_{\text{\test}} = \max_{s} r_{\text{\test}}(s)$. Interestingly, against a worst-case distribution $\mu = \delta_n$ (all probability mass concentrated on $n$) $r_\text{\test}$ exactly attains the bound of Equation~\ref{eq:saks}. As a consequence, following the conclusions from the analysis of \solve in Section~\ref{sec:binarygames}, the \textit{F-games} are typically harder problems than the games generated under \textit{standard model}, even for discrete-valued trees. That makes $r_{\text{\test}}$ an interesting and easy-to-compute quantity to gauge the difficulty induced by the choice of a distribution $\mu$. We use this property in Section~\ref{sec:exp} as a measure of difficulty: if $r_{\text{\test}}$ is large (resp. small) the game-trees are considered hard (resp. easy). In the following section, we characterize the average-case complexity of \alphabeta and compare its branching factor to that of \test.

\subsection{Analysis of ALPHA-BETA} 
\begin{wrapfigure}[13]{r}{0.56\textwidth} 
\vspace{-1.25cm}
\begin{algorithm}[H]
\SetAlgoLined
\DontPrintSemicolon

\KwIn{Current node \textit{N}, search depth $h$, lower-bound $\alpha$, upper-bound $\beta$}
\KwOut{Value of node \textit{N}.}

\lIf{$h = 0$}{    \Return{N.value} }

$best \leftarrow -\infty$\;
\ForEach{N' in N.children}{
    $value \leftarrow  -\text{ALPHABETA}(N', h-1, -\beta, -\alpha)$ 

    $best \leftarrow \max(best,value)$ 
    
    \lIf{$best \geq \beta$}{        \textbf{break}     }

    $\alpha \leftarrow \max(\alpha, best)$\;

}
\Return{$best$}\;

\caption{ALPHABETA($N, h, \alpha, \beta$), \textit{negamax} form}\label{alg:alphabeta}
\end{algorithm}

\end{wrapfigure}
\paragraph{Algorithm description} The \alphabeta algorithm improves upon classical full negamax search by pruning branches that cannot influence the root value. It tracks two evolving bounds: $\alpha$ (the worst-case guarantee for the maximizing player) and $\beta$ (the best-case allowance for the minimizing player). As described in Algorithm~\ref{alg:alphabeta}, upon a child evaluation, achieved through a recursive call with negated parameters $\alpha'=-\beta$ and $\beta' = -\alpha$, \alphabeta updates the current best value and early terminates whenever it exceeds $\beta$, and it potentially updates $\alpha$ for next sibling evaluation. While typically invoked with a full-window ($\alpha=-\infty$, $\beta=+\infty$) to compute the root value $x$ exactly, \alphabeta can also operate with bounded intervals, in this case if $x$ is not comprised in $[\alpha, \beta]$, it will return a certificate value (like \test) asserting whether $x$$\geq$$\beta$ or $x$$\leq$$\alpha$. A well-known~\cite{plaat1994new} connection to \test appears with a so-called \textit{null-window} ($\alpha = s-1$, $\beta = s$), in this case, \alphabeta becomes functionally equivalent to \textsc{test($s$)}, incurring the same complexity and producing identical certificates. 
\paragraph{Complexity analysis} We define $I_{\text{\ab}}^{x,\alpha, \beta}(h, c)$ the average-case complexity of \alphabeta called with parameters $\alpha$$<$$\beta$ and using same notations as before. We are interested in  $\smash{I_{\text{\ab}}(h) = \mathbb{E}_{X\sim\mu}[I_{\textsc{ab}}^{X, -n,n}(h)}]$, the average complexity of \alphabeta called with a full-window, and in $r_\textsc{ab}$, its branching factor. The derivation of this complexity follows closely the one for \test and yields a similar system of recursive equations. For a node with value $X'$, where \alphabeta is called with parameters $\alpha$ and $\beta$, the main differences with \test are: (1) the cutoff condition now becomes $-X'$$<$$\beta$, (2) the recursive calls to \alphabeta are made with parameters $-\beta$ and $-\alpha$ and (3) unlike in \test, further calls at the same level may use an updated value of $\alpha$ if the current child has the best value encountered so far. Assuming $J_\textsc{ab}$ follows a similar definition to that of $J_\textsc{test}$:
\begin{align}
    \vspace{-2cm}
    \begin{aligned}
        I_{\text{\ab}}^{x,\alpha, \beta}(h, c) = 
        \frac{1}{c} \biggl[ I_{\text{\ab}}^{-x,-\beta,-\alpha}(h-1, b) + \mathds{1}_{\{x<\beta\}} J^{x,\max(\alpha,x), \beta}_{\text{\ab}}(h,c-1) &\biggr]&& \\
      + \frac{c-1}{c} \underset{\substack{X' \sim \mu \\ X'\geq-x }}{\mathbb{E}}
        \biggl[I_{\text{\ab}}^{X',-\beta,-\alpha}(h-1, b) + \mathds{1}_{\{-X'<\beta\}}I_{\text{\ab}}^{x,\max(\alpha,-X'), \beta}(h, c-1) &\biggr],
    \end{aligned}
    \label{equation:ab_complexity} \\
    \textup{and }J_{\text{\ab}}^{x,\alpha,\beta}(h, c) = 
    \underset{\substack{X' \sim \mu \\ X'\geq-x }}{\mathbb{E}}
    \biggl[I_{\text{\ab}}^{X',-\beta, -\alpha}(h-1, b) + \mathds{1}_{\{-X'<\beta\}}J_{\text{\ab}}^{x,\max(\alpha, -X'), \beta}(h, c-1) \biggr]
    \label{equation:ab_complexity_j}.
    \vspace{-1cm}
\end{align}
This system of equations resembles Equations~\ref{eq:i_test_recurrence}~and~\ref{eq:j_test_recurrence} and allows us to run comprehensive numerical simulations, which lead to the remarkable observation that \test and \alphabeta share the same branching factor. We find that the equality of the branching factor holds theoretically, as indicated in the following theorem (proof in Appendix~\ref{app:ab_proof}):

\newcommand{\thmtextd}{\alphabeta called with a full window $\{-n,...,n\}$ is more efficient than the approach consisting in testing every possible value with the \test procedure (i.e., \testbrute) and \alphabeta and \test share the same asymptotic branching factor, in the precise sense that for $h\geq 0$ and $x\in \{-n,\ldots,n\}$:
{\begin{align*}
    I_{\text{\ab}}^{x,-n, n}(h) \leq \sum_{s=-n+1}^{n} I_{\text{\test}}^{x,s}(h) \quad \text{and} \quad r_{\text{\ab}} = r_{\text{\test}}.
\end{align*}}}
\newtheorem*{Td}{Theorem~\ref{th:ab51}}

\begin{theorem}
\thmtextd
\label{th:ab51}
\end{theorem}
\vspace{-3mm}
This result shows that there is no asymptotic gain of using \alphabeta over a simple \testbrute approach, that calls \test $2n$ times. Moreover, numerical results, presented in Section~\ref{sec:exp} suggest that these two algorithms present a deeper identical behavior: both for the asymptotic limit and the multiplicative constant characterizing the convergence rate. This result is quite remarkable, since the question answered by \alphabeta---determining the precise value of the game---intuitively seems to be much harder than the question answered by \test, that only solves a binary problem. In the next section, we conduct a similar average-case analysis of the \scout algorithm.

\subsection{Analysis of SCOUT} 

\begin{wrapfigure}[13]{r}{0.53\textwidth} 
\vspace{-1.35cm}
\begin{algorithm}[H]
\SetAlgoLined
\DontPrintSemicolon

\KwIn{Current node \textit{N}, search depth $h$, lower-bound $\alpha$, upper-bound $\beta$}
\KwOut{Value of node \textit{N}}

\lIf{$h = 0$}{\Return{N.value} }
\lIf{$\alpha \geq \beta$}{\Return{$\alpha$} }

\ForEach{N' in N.children}{
    $test \leftarrow -\textup{TEST}(N', h-1, -\alpha)$
    
    \If{$test > \alpha$  }{

        $\alpha \leftarrow  -\text{SCOUT}(N', h-1, -\beta, -\alpha-1)\footnotemark{}$
    }
    \lIf{$\alpha \geq  \beta$}{\textbf{break} }
}
\Return{$\alpha$}\;

\caption{SCOUT($N, h,\alpha,\beta$)}\label{alg:scout}
\end{algorithm}
\end{wrapfigure}

\paragraph{Algorithm description} \scout incorporates the \test algorithm into a procedure similar to \alphabeta, described in Algorithm~\ref{alg:scout}. Before evaluating any node, \scout first performs a call to \test to check whether the child’s value strictly exceeds $\alpha$. Only if this test returns true, indicating potential for improvement, does \scout proceed to evaluate the node in full, by a negated recursive call to itself, and updates $\alpha$ to the current (higher) value. Like \alphabeta, it is most often called with $\alpha = -\infty$, but can also be called with any value of $\alpha$ and potentially a parameter $\beta$ too (triggering a cutoff whenever $\alpha$$\geq$$\beta$). At first glance, this approach may appear inefficient: when a \test returns true, subsequent evaluations revisit some leaf nodes already examined during the threshold check. However, previous experimental results~\cite{campbell1983comparison,muszycka1985empirical} suggest that the waste incurred by \scout's reevaluation of some nodes is not substantial. Furthermore, improved variants of \scout---Principal Variation Search (PVS) and NegaScout \cite{reinefeld1983improvement, campbell1983comparison}---are still used in modern game engines~\cite{stockfish}. In this work we focus on the original \scout algorithm, which has a simpler formal analysis: unlike its improved variants, it does not use the certificate value $v$ but only the boolean outcome $\mathds{1}_{v>\alpha}$ of the \test procedure. We believe that PVS and NegaScout improvements should not change the asymptotic behavior of \scout, and we leave their formal analysis for future work.
\vspace{-3mm}
\paragraph{Complexity analysis} Mirroring the \alphabeta analysis, we define $I_{\text{\scout}(h,c)}^{x,\alpha,\beta}$ as the complexity of \scout conditioned on $x$ and $\alpha$$<$$\beta$ and we are interested in the complexity and branching factor for a full-window $\smash{I_{\text{\scout}(h)} = \mathbb{E}_{X'\sim\mu}[I_{\textsc{scout}(h)}^{X', -n,n}}]$ and $r_{\text{\scout}}$. We can write the following recursive equations:\footnotetext{Though this is not standard in the description of \scout, it seems that we can freely increment $\alpha$ by one each time a strict test is proven true, because if $x > \alpha \text{ then }  x\in[\alpha+1,\beta]$, facilitating our formal analysis of \scout.}

\begin{align}
    \begin{split}
        {I_{\text{\scout}(h,c)}^{x,\alpha,\beta}}= \frac 1 c \bigg[{I_{\text{\test}(h-1,b)}^{-x,-\alpha}} +\mathds{1}_{\{ \alpha< x \}}{I_{\text{\scout}(h-1,b)}^{-x, -\beta, -\alpha-1}} + \mathds{1}_{\{x<\beta\}}\ {J_{\text{\scout}(h,c-1)}^{x, \max{(\alpha, x)},\beta}}\bigg]+ \\   
         \frac {c-1}c \mathbb E_{X'\geq-x}\bigg[{I_{\text{\test}(h-1, b)}^{X', -\alpha}}+\mathds{1}_{\{\alpha < -X'\}}I^{X', -\beta,-\alpha-1}_{\text{\scout}(h-1, b)}+\mathds{1}_{\{-X'<\beta\}}I^{x, \max(\alpha, -X'),\beta}_{\text{\scout}(h,c-1)}\bigg],
    \end{split}
    \label{eq:i_scout} \\
    \begin{split}
    J_{\text{\scout}(h,c)}^{x,\alpha,\beta}=\mathbb E_{X'\geq-x}\bigg[{I_{\text{\test}(h-1, b)}^{X', -\alpha}}+\mathds{1}_{\{\alpha<-X'\}}I^{X', -\beta,-\alpha-1}_{\text{\scout}(h-1, b)}+\mathds{1}_{\{-X'<\beta\}}J^{x, \max(\alpha, -X'),\beta}_{\text{\scout}(h,c-1)}\bigg]
    \end{split}
    \label{eq:j_scout}
\end{align}

This defines a system very similar to that of \alphabeta, and extensive numerical studies that we conducted suggest that the branching factor of \scout coincides with that of \test and \alphabeta. We state here an analogous result to the one we obtained for \alphabeta (proof in Appendix~\ref{app:scout_proof}):

\newcommand{\thmtextc}{
\scout called with a full window $\{-n,...,n\}$ is more efficient than the approach consisting in testing every possible value with the \test procedure (i.e., \testbrute). For $h\geq 0$ and $x\in \{-n,...,n\}$:\vspace{-1mm}
\begin{align*}
    I_{\text{\scout}}^{x,-n, n}(h) \leq \sum_{s=-n+1}^{n} I_{\text{\test}}^{x,s}(h) \quad \text{and} \quad r_{\text{\scout}} \leq r_{\text{\test}}.
\end{align*}
\vspace{-1mm}
}

\newtheorem*{Tc}{Theorem~\ref{th:scout}}

\begin{theorem}
\thmtextc\label{th:scout}
\end{theorem}
This result is weaker than Theorem~\ref{th:ab51}, in the sense that we did not manage to prove that \scout is asymptotically equivalent to \test. The missing part is to show that $r_{\text{\scout}} \geq r_{\text{\test}}$, which is suggested by numerical simulations.

In the next section, we experimentally compare all presented algorithms for deep trees and different parametrizations of the \textit{forward model}.

\section{Finite-depth numerical analysis}

\begin{figure}[h]

\centering
\includegraphics[width=1.\textwidth]{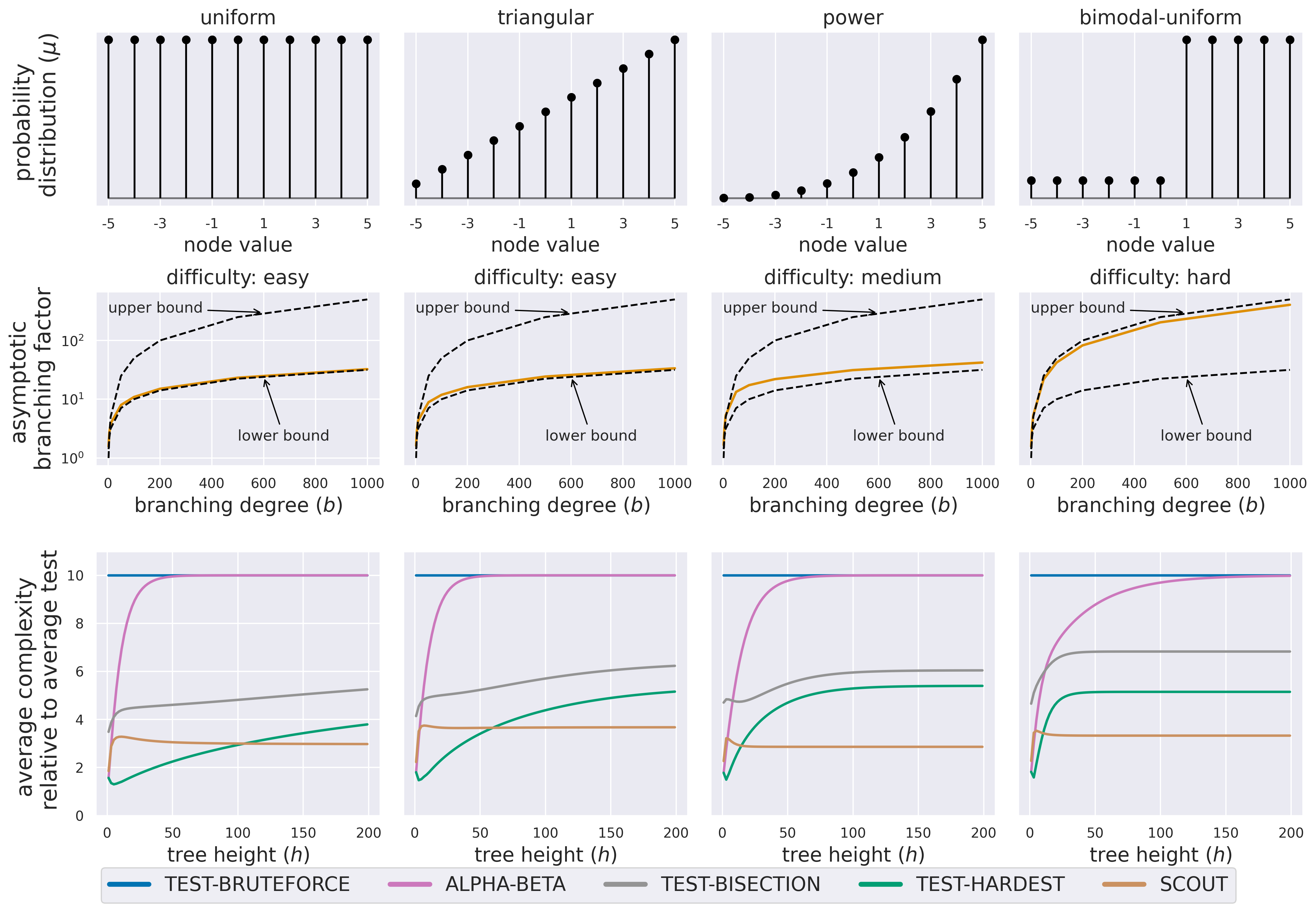}
\vspace{-0.6cm}
\caption{Finite-depth comparison of average-complexities. Each column corresponds to a parametrization of the \textit{forward model}, consisting in the choice of the distribution $\mu$. (Top) Probability mass function of $\mu$ (Middle) Difficulty of the generated instances induced by the choice of $\mu$, measured as the dependency of the branching factor $r$ to the branching degree $b$. (Bottom) Relative average-complexity w.r.t. the average performance of the algorithm \test, this allows us to visualize the sub-exponential multiplicative factor as a function of the height $h$ (with $b=10$, $n=5$). For each parametrization \alphabeta demonstrates the same convergence rate as \testbrute. Conversely, \scout and \testbi exhibit faster convergence, and remarkably \scout even seems to incur a smaller multiplicative constant than \textsc{test-hardest}.\vspace{-0.5cm}}
\label{fig:results}
\end{figure}

\label{sec:exp}
In Section~\ref{sec:alphabeta_complexit} we have established that \test and \alphabeta shared the same asymptotic branching factor, which also coincides with that of \scout numerically. To gain more insights into the behavior of these algorithms, and especially better understand their sub-exponential convergence rates, we conduct a finite-depth experimental analysis, using the recursive equations derived earlier in this work. 
\paragraph{Baselines} In addition to \alphabeta and \scout, we choose to consider two \test-based baselines: (1) \testbrute, which comprehensively applies \test for every threshold $s \in \{-n+1,\ldots,n\}$ and (2) \testbi, which uses a binary search approach to reduce the number of \test trials from $\mathcal{O}(n)$ to $\mathcal{O}(\log n)$. We also introduce the \textsc{test-hardest} baseline which corresponds to the most expensive call to \test across all threshold $s$ values.
\vspace{-2mm}
\paragraph{Parametrization}
To best compare the algorithms, we design diverse instantiations of the \textit{forward model}, with different distributions $\mu$. Ideally, we want these parametrizations to cover a wide range of problem difficulty, from the simplest to the hardest case. We suggest that the asymptotic branching factor (common to all algorithms) $r$ can be interpreted as a measure of the intrinsic game's difficulty. We propose an arbitrary and informal way of classifying the difficulty: for a given distribution $\mu$, value range $n$ and branching degree $b$, if the branching factor $r$ is close to $\smash{\sqrt{b}}$ (the lower bound) then we classify it as easy, similarly if it's close to $b$ (the upper bound) we classify it as hard, otherwise we classify it as medium. For the specific choice of these distributions we start with a uniform distribution---a natural choice that maximizes the diversity (entropy) of root values. Then, we choose distributions that assign increasing probability mass on positive atoms, because we have found empirically that this leads to produce problems of increased difficulty. The intuition is that a node with value $x$ will draw children in $\{-x,..., n\}$, thus the larger $x$, the wider the interval, and the higher the diversity of node values in the generated tree. As displayed in the top row of Figure~\ref{fig:results}, we choose a triangular distribution, a power-law (cubic) distribution and finally a bimodal-uniform distribution with more than $1/b$ mass concentrated on positive atoms (following a criterion similar to that of Theorem~\ref{th:solve1b}).
\vspace{-2mm}
\paragraph{Results} We present in Figure~\ref{fig:results}, the results of the finite-depth average-case complexity comparison. Each column in the figure represents, top to bottom, the probability distribution $\mu$ used, the difficulty this choice induces in the generated game-trees and finally the actual performance of the compared algorithms. To best compare the relative performance of algorithms, we divide the average-complexity by that of \test, computed as an average of all possible values of threshold $s$. This allows us to easily distinguish the algorithm transient regimes by focusing on the sub-exponential multiplicative factor and ignoring the mechanical effect of the complexity increasing with the height of the tree.
A striking result appears: \alphabeta consistently ends up as the worst algorithm across all evaluation setups as well as for every tree size. Remarkably, it seems to mirror the performance of the naive \testbrute baseline for very deep trees. It may suggest that through recursive calls, the $\alpha$ and $\beta$ parameters of \alphabeta rapidly reduce to null-window situations where $\alpha=\beta-1$, and end up comprehensively testing all possible threshold values, thus becoming equivalent to the \testbrute approach. This might be a symptom of a deeper asymptotic equivalence between \alphabeta and \testbrute, which seem to behave identically asymptotically as well as in their sub-exponential constant factor.

Conversely, \scout seems to consistently outperform \alphabeta and to achieve the best performance for every game difficulty. Unlike the conclusions drawn in a fixed-depth analysis under the \textit{standard model}~\cite{pearl1980asymptotic}, it supports the practical superiority of \scout over \alphabeta as hinted in multiple numerical studies~\cite{campbell1983comparison,muszycka1985empirical}. \testbi also displays a strong performance, we believe this comes from its conceptual similarity to the MTD(f) algorithm, whose practical superiority over \alphabeta have been suggested in the past~\cite{plaat1996best}.  Interestingly enough, \scout seems to even outperform \textsc{test-hardest} for deeper trees. This result is counter-intuitive since \scout's cost mostly comes from calls to \test. This probably suggests that \scout behaves like an adaptive version of \test, which updates the threshold value according to the values encountered, unlike \test, which has a fixed threshold value, making it more sensitive to worst-case situations.

\section{Related works}
\label{sec:related}

The average-case analysis of minimax algorithms originated from studying \alphabeta under the \textit{standard model} with independent leaf values \cite{fuller1973analysis, baudet1978branching}, where it achieves optimality \cite{pearl1982solution,tarsi1983optimal}. Subsequent algorithms, \textsc{scout}~\cite{pearl1980asymptotic}, MTD(f)~\cite{plaat1994new} and PVS (sometimes called NegaScout)~\cite{reinefeld1983improvement} were proven asymptotically equivalent. Critiques of the \textit{standard model} highlight its unrealistic independence assumption~\cite{knuth1975analysis, nau1982lastplayer}. Alternative models introduce ancestor dependencies, e.g., the \textit{incremental model} maintains a heuristic value at each node, and defines leaf values as the sum of values along the path to the root node. Only limited settings of this formulation of the \textit{incremental model} have been successfully analyzed~\cite{devroye1996random, newborn1977efficiency} and a more general analysis is yet to be proven feasible. Prior works also explore the idea of forward tree generation~\cite{scheucher1998benefits,furtak2009minimum}, where values are generated top-down like in our \textit{forward model}. The closest model to ours is the Prefix Value Game Tree Model~\cite{furtak2009minimum}, it is an instantiation of the \textit{incremental model} where heuristic values of nodes are computed additively from their parents, and one child is always attributed a zero increment, this resembles our \textit{forward model} where one child inherits the value of the parent. However, to our knowledge, this model has not yet been successfully theoretically analyzed before. Moreover, we are confident that the analysis techniques we developed in this work could be very simply adapted to it.

 Another popular approach in modern game solving is Monte-Carlo Tree Search (MCTS)~\cite{kocsis2006bandit, delattre2019monte}, however, a common theoretical average-complexity analysis under a shared framework with minimax-based game-solving algorithms is currently missing. This is explained by technical reasons. MCTS algorithms only offer asymptotic convergence guarantees (vs. deterministic guarantees) and their convergence rate can be doubly-exponential in the size of the tree \cite{coquelin2007bandit,orseau2024super} instead of at most exponential for minimax algorithms. Furthermore, while we believe extending our analysis to an MCTS-based algorithm is possible in principle, it presents a formidable technical challenge as the algorithm's states depend on visit counts and empirical rewards at each node, this would make the equations substantially more involved, and we view it as an important direction for future research. Finally, although MCTS-based algorithms are state-of-the-art in multiple settings, exact minimax algorithms still remain practically useful in chess engines~\cite{stockfish} and hybrid (RL+solving) approaches~\cite{cohen2020minimax}.

\section{Discussion and limitations}
\label{sec:discussion}  

\paragraph{Discussion} In this work, we introduced the \textit{forward model}, a simple synthetic game model which provably addresses limitations of previous models while retaining some form of analytical tractability. For binary-valued trees, we established a closed-form expression for \solve's average-complexity and branching factor, and showed that the \textit{F-games} were of adjustable difficulty and could span from maximally easy to maximally hard instances. For discrete-valued trees, we characterized the behavior of the \test, \alphabeta and \scout algorithms with equations, allowing a fast and efficient computation of the complexity and asymptotic branching factor, intractable with Monte-Carlo simulations. Unlike previous analysis under the \textit{standard model}, we didn't manage to find closed-form expressions for the branching factors of all studied algorithms and we leave this open for future work. However, we established that \test and \alphabeta share the same branching factor, and we hypothesize that \scout shares it as well. This property was numerically confirmed by extensive numerical experiments, which further revealed that \alphabeta incurred a larger sub-exponential factor than other approaches, suggesting that it is a poor baseline for practical scenarios.

\paragraph{Limitations} Our work focused on discrete-valued trees, which in our opinion best represent real-world games, and are numerically cheaper to solve. That being said, the equations we derived for the complexity analysis of \test, \alphabeta and \scout are, with small modifications, applicable to continuous values, opening avenues for an extended analysis of \textit{F-games}. An important disclaimer is that we do not claim that our model accurately represents real-world games dynamics: it suffers from several important limitations that are consequences of its simplicity. First, the model only generates uniform $b$-ary trees of constant depth, which does not reflect the variable branching degrees and depths found in real games. Extending the analysis to more complex tree structures is a direction for future work. Second, the conditional distribution used to generate child values is the same for every node in the tree. It depends only on the parent's value, not on the broader game state. Furthermore, the method of ensuring the minimax constraint (truncating the distribution) is just one of many possibilities; other, more complex transformations could be valid but may also be harder to analyze. Regarding the choice of algorithms, we focused on \test, \scout and \alphabeta, but an analysis of MTD(f) and PVS would be an insightful extension, though they may prove more difficult to conduct since they require modeling the distribution of the test certificate value (and not only its binary outcome). However, we believe that \scout is a good proxy for PVS, and that PVS optimizations should not change asymptotic properties. Similarly, we think \testbi is good proxy for MTD(f) where instead of choosing the next threshold using the certificate value of the last iterate---as in MTD(f)---we select it with a simple bisection rule. Of course, this is only speculation, which is why we believe it’s important that PVS and MTD(f) are properly analyzed in the future under the \textit{forward model}.

\section{Broader impact}
While our work focuses on classical deterministic game-solving algorithms, we believe it contributes to ongoing discussions at the intersection of learning, planning, and decision-making. Search remains a core ingredient of modern AI systems: top-performing agents such as AlphaZero combine learning with planning, yet the theoretical understanding of deterministic search methods (e.g., AlphaBeta) remains limited. Our analysis introduces a simple average-case model that incorporates structural dependencies in search trees, bridging a gap in the theoretical understanding of deterministic solvers. By providing a clearer picture of their expected behavior, our work may help guide the design of hybrid systems that integrate learning and deterministic planning.
\newpage

\bibliographystyle{plainnat}
\bibliography{sample}

\begin{thebibliography}{24}
\providecommand{\natexlab}[1]{#1}
\providecommand{\url}[1]{\texttt{#1}}
\expandafter\ifx\csname urlstyle\endcsname\relax
  \providecommand{\doi}[1]{doi: #1}\else
  \providecommand{\doi}{doi: \begingroup \urlstyle{rm}\Url}\fi

\bibitem[Baudet(1978)]{baudet1978branching}
G{\'e}rard~M Baudet.
\newblock On the branching factor of the alpha-beta pruning algorithm.
\newblock \emph{Artificial Intelligence}, 10\penalty0 (2):\penalty0 173--199,
  1978.

\bibitem[Campbell and Marsland(1983)]{campbell1983comparison}
Murray~S Campbell and T.~Anthony Marsland.
\newblock A comparison of minimax tree search algorithms.
\newblock \emph{Artificial Intelligence}, 20\penalty0 (4):\penalty0 347--367,
  1983.

\bibitem[Cohen-Solal and Cazenave(2020)]{cohen2020minimax}
Quentin Cohen-Solal and Tristan Cazenave.
\newblock Minimax strikes back.
\newblock \emph{arXiv preprint arXiv:2012.10700}, 2020.

\bibitem[Coquelin and Munos(2007)]{coquelin2007bandit}
Pierre-Arnaud Coquelin and R{\'e}mi Munos.
\newblock Bandit algorithms for tree search.
\newblock \emph{arXiv preprint cs/0703062}, 2007.

\bibitem[Delattre and Fournier(2019)]{delattre2019monte}
Sylvain Delattre and Nicolas Fournier.
\newblock On monte-carlo tree search for deterministic games with alternate
  moves and complete information.
\newblock \emph{ESAIM: Probability and Statistics}, 23:\penalty0 176--216,
  2019.

\bibitem[developers()]{stockfish}
The~Stockfish developers.
\newblock Stockfish.
\newblock URL
  \url{https://github.com/official-stockfish/Stockfish/blob/master/AUTHORS}.

\bibitem[Devroye and Kamoun(1996)]{devroye1996random}
Luc Devroye and Olivier Kamoun.
\newblock Random minimax game trees.
\newblock In \emph{Random Discrete Structures}, pages 55--80. Springer, 1996.

\bibitem[Fuller et~al.(1973)Fuller, Gaschnig, Gillogly,
  et~al.]{fuller1973analysis}
Samuel~H Fuller, John~G Gaschnig, JJ~Gillogly, et~al.
\newblock \emph{Analysis of the alpha-beta pruning algorithm}.
\newblock Department of Computer Science, Carnegie-Mellon University, 1973.

\bibitem[Furtak and Buro(2009)]{furtak2009minimum}
Timothy Furtak and Michael Buro.
\newblock Minimum proof graphs and fastest-cut-first search heuristics.
\newblock In \emph{IJCAI}, volume~9, pages 492--498, 2009.

\bibitem[Knuth and Moore(1975)]{knuth1975analysis}
Donald~E Knuth and Ronald~W Moore.
\newblock An analysis of alpha-beta pruning.
\newblock \emph{Artificial intelligence}, 6\penalty0 (4):\penalty0 293--326,
  1975.

\bibitem[Kocsis and Szepesv{\'a}ri(2006)]{kocsis2006bandit}
Levente Kocsis and Csaba Szepesv{\'a}ri.
\newblock Bandit based monte-carlo planning.
\newblock In \emph{European conference on machine learning}, pages 282--293.
  Springer, 2006.

\bibitem[Muszycka and Shinghal(1985)]{muszycka1985empirical}
Agata Muszycka and Rajjan Shinghal.
\newblock An empirical comparison of pruning strategies in game trees.
\newblock \emph{IEEE transactions on systems, man, and cybernetics}, \penalty0
  (3):\penalty0 389--399, 1985.

\bibitem[Nau(1982{\natexlab{a}})]{nau1982investigation}
Dana~S Nau.
\newblock An investigation of the causes of pathology in games.
\newblock \emph{Artificial Intelligence}, 19\penalty0 (3):\penalty0 257--278,
  1982{\natexlab{a}}.

\bibitem[Nau(1982{\natexlab{b}})]{nau1982lastplayer}
Dana~S. Nau.
\newblock The last player theorem.
\newblock \emph{Artificial Intelligence}, 18\penalty0 (1):\penalty0 53--65,
  1982{\natexlab{b}}.
\newblock ISSN 0004-3702.
\newblock \doi{https://doi.org/10.1016/0004-3702(82)90010-8}.
\newblock URL
  \url{https://www.sciencedirect.com/science/article/pii/0004370282900108}.

\bibitem[Newborn(1977)]{newborn1977efficiency}
Monroe~M. Newborn.
\newblock The efficiency of the alpha-beta search on trees with
  branch-dependent terminal node scores.
\newblock \emph{Artificial Intelligence}, 8\penalty0 (2):\penalty0 137--153,
  1977.

\bibitem[Orseau and Munos(2024)]{orseau2024super}
Laurent Orseau and Remi Munos.
\newblock Super-exponential regret for uct, alphago and variants.
\newblock \emph{arXiv preprint arXiv:2405.04407}, 2024.

\bibitem[Pearl(1980)]{pearl1980asymptotic}
Judea Pearl.
\newblock Asymptotic properties of minimax trees and game-searching procedures.
\newblock \emph{Artificial Intelligence}, 14\penalty0 (2):\penalty0 113--138,
  1980.

\bibitem[Pearl(1982)]{pearl1982solution}
Judea Pearl.
\newblock The solution for the branching factor of the alpha-beta pruning
  algorithm and its optimality.
\newblock \emph{Communications of the ACM}, 25\penalty0 (8):\penalty0 559--564,
  1982.

\bibitem[Plaat et~al.(1994)Plaat, Schaeffer, Pijls, and de~Bruin]{plaat1994new}
Aske Plaat, Jonathan Schaeffer, Wim Pijls, and Arie de~Bruin.
\newblock A new paradigm for minimax search.
\newblock 1994.

\bibitem[Plaat et~al.(1996)Plaat, Schaeffer, Pijls, and
  De~Bruin]{plaat1996best}
Aske Plaat, Jonathan Schaeffer, Wim Pijls, and Arie De~Bruin.
\newblock Best-first fixed-depth minimax algorithms.
\newblock \emph{Artificial Intelligence}, 87\penalty0 (1-2):\penalty0 255--293,
  1996.

\bibitem[Reinefeld(1983)]{reinefeld1983improvement}
Alexander Reinefeld.
\newblock An improvement to the scout tree search algorithm.
\newblock \emph{ICGA Journal}, 6\penalty0 (4):\penalty0 4--14, 1983.

\bibitem[Saks and Wigderson(1986)]{saks1986probabilistic}
Michael Saks and Avi Wigderson.
\newblock Probabilistic boolean decision trees and the complexity of evaluating
  game trees.
\newblock In \emph{27th Annual Symposium on Foundations of Computer Science
  (sfcs 1986)}, pages 29--38. IEEE, 1986.

\bibitem[Scheucher and Kaindl(1998)]{scheucher1998benefits}
Anton Scheucher and Hermann Kaindl.
\newblock Benefits of using multivalued functions for minimaxing.
\newblock \emph{Artificial Intelligence}, 99\penalty0 (2):\penalty0 187--208,
  1998.

\bibitem[Tarsi(1983)]{tarsi1983optimal}
Michael Tarsi.
\newblock Optimal search on some game trees.
\newblock \emph{Journal of the ACM (JACM)}, 30\penalty0 (3):\penalty0 389--396,
  1983.

\end{thebibliography}

\newpage
\appendix








\section{Proof material}
In this section we provide detailed proofs for the results established in the paper, that were not included in the main text due to space and readability constraints.
\subsection{Proof of Theorem~\ref{th:solve1b}}
\label{app:solve_proof}
We first recall the result:
\begin{Ta}
    \thmtexta
\end{Ta}
\begin{proof}
Recall that
\begin{align*}
    t(q,b) = \sum_{k=1}^{b-1} \frac{1 + (b - k - 1)q}{b} k (1 - q)^k.
\end{align*}

By analyzing the sign of $\frac {\partial t}{\partial q}(q,b)$ and showing it's non-positive for $q\in [0,1]$ we can establish that $t$ is monotonically decreasing in $q$ on  $[0,1]$. Then, we only have to show that $t_b = t(q=1/b, b) = \mathcal{O}(b)$, for the result to be true for all $q\in[0,\frac 1 b]$. First we separate positive and negative sums:
\begin{align}
    t_b &= \sum_{k=1}^{b-1} \frac{1 + (b - k - 1)\frac 1 b}{b} k (1 - \frac 1 b)^k \\
    &= \frac 2 b \sum_{k=1}^{b-1}  k (1 - \frac 1 b)^k - \frac 1 {b^2} \sum_{k=1}^{b-1} k(k+1)  (1 - \frac 1 b)^k.
\end{align}

By expanding both sums (with geometrical sum expansions) and using the fact that:
\begin{equation}
    (1-\frac 1 b)^b=e^{-1}+ o(1),
\end{equation}
we can show that:
\begin{align}
    &\frac 2 b\sum_{k=1}^{b-1} k(1-\frac 1 b)^k = (2-4e^{-1})b+o(b) \\
    \textnormal{and }&\frac 1 {b^2}\sum_{k=1}^{b-1} k(k+1)(1-\frac 1 b)^k = (2-5e^{-1})b+o(b),
\end{align} 
allowing us to conclude that $t_b = e^{-1}b +  o(b) = \mathcal{O}(b)$. This terminates the proof. As a side note, if we consider instead $q=\frac 1 {b^a}$ for $a>1$, we can show that $t_b = \frac b 2 + o(b)$, which is slightly higher than $e^{-1} b$ (because $e^{-1}\approx 0.37 < \frac 1 2$) and asymptotically reaches the same constant as in Equation~\ref{eq:saks}.
\end{proof}

\subsection{Proof of Theorem~\ref{th:ab51}}
\label{app:ab_proof}
We start by proving an intermediate property, useful for proving Theorem~\ref{th:ab51}.

\newcommand{\thmtextb}{
    Evaluating \alphabeta with a given window is always more efficient than splitting this window into two sub-windows and evaluating these sub-windows separately. 
    
    For $h\geq 0$ and $x\in [-n,n]$, for all $\alpha,\beta \in \mathbb N$ such that $\alpha < \beta -1$ and for all $\gamma \in \mathbb N$ such that $\alpha < \gamma < \beta$:
    \begin{align*}
           I_{\text{\ab}}^{x,\alpha, \beta}(h) \leq I_{\text{\ab}}^{x,\alpha, \gamma}(h) + I_{\text{\ab}}^{x, \gamma,\beta}(h).
    \end{align*}}
\newtheorem*{Tb}{Proposition~\ref{prop:alphabeta}}

\begin{proposition}
    \thmtextb
    \label{prop:alphabeta}
\end{proposition}
\begin{proof}

Let's recall the expression of the complexity of \alphabeta $I_{\text{\ab}}^{x,\alpha, \beta}(h, c)$ and it's auxiliary function $J_{\text{\ab}}^{x,\alpha,\beta}(h, c)$:
\begin{equation}
\begin{split}
    I_{\text{\ab}}^{x,\alpha, \beta}(h, c)=
\frac 1 c \Biggl[ I_{\text{\ab}}^{-x,-\beta,-\alpha}(h-1, b) + \mathds{1}_{\{x<\beta\}}\ J^{x,\max(\alpha,x), \beta}_{\text{\ab}}(h,c-1) \Biggr]\\ + \frac {c-1}c \underset{\substack{X' \sim \mu \\ X'\geq-x }}{\mathbb E}\Biggl[I_{\text{\ab}}^{X',-\beta,-\alpha}(h-1, b)+\mathds{1}_{\{-X'<\beta\}}I_{\text{\ab}}^{x,\max(\alpha,-X'), \beta}(h, c-1) \Biggr],
\end{split}
\end{equation}
\begin{equation}
J_{\text{\ab}}^{x,\alpha,\beta}(h, c)=
\underset{\substack{X' \sim \mu \\ X'\geq-x }}{\mathbb E}\Biggl[I_{\text{\ab}}^{X',-\beta, -\alpha}(h-1, b)+\mathds{1}_{\{-X'<\beta\}}J_{\text{\ab}}^{x,\max(\alpha, -X'), \beta}(h, c-1) \Biggr].
\end{equation}

We conduct the proof for the function $I$ and $J$ altogether, \textit{i.e.} we want to prove that for all $\alpha <\gamma< \beta$ and all integers $h,c$ we have:
$$I_{\text{\ab}}^{x,\alpha, \beta}(h,c) \leq I_{\text{\ab}}^{x,\alpha, \gamma}(h,c) + I_{\text{\ab}}^{x, \gamma,\beta}(h,c) $$
and
$$J_{\text{\ab}}^{x,\alpha, \beta}(h,c) \leq J_{\text{\ab}}^{x,\alpha, \gamma}(h,c) + J_{\text{\ab}}^{x, \gamma,\beta}(h,c) $$.

Let's proceed by double induction on the integers $h$ and $c$.

\paragraph{Base cases}  for all $h$, $I_{\text{\ab}}^{x,\alpha,\beta}(h, 0)=0$ and  $I_{\text{\ab}}^{x,\alpha, \gamma}(h,0) + I_{\text{\ab}}^{x, \gamma,\beta}(h,0) =0$. Similarly for all $c>0$, $I_{\text{\ab}}^{x,\alpha,\beta}(0, c)=1$ and  $I_{\text{\ab}}^{x,\alpha, \gamma}(0,c) + I_{\text{\ab}}^{x, \gamma,\beta}(0,c) =2$. It follows identically for J. So the base cases hold.

\paragraph{Induction step} now we assume that the property holds for $(h, c-1)$ and $(h-1,b)$. We'll detail here the induction step for the function $J$, as it is less cumbersome to write, but the proof for $I$ follows the exact same steps. First, let's remark that the expectation involves a sum of terms, and let's try to prove the inequality holds term by term. Let $X'\geq -x$, first if $\mu(X')=0$, the inequality holds trivially, so we consider without loss of generality $\mu(X')>0$. Let's define:
$$A = I_{\text{\ab}}^{X',-\beta, -\alpha}(h-1, b)+\mathds{1}_{\{-X'<\beta\}}J_{\text{\ab}}^{x,\max(\alpha, -X'), \beta}(h, c-1)$$
and
\begin{align*}
    \begin{split}
    B = I_{\text{\ab}}^{X',-\gamma, -\alpha}(h-1, b)+\mathds{1}_{\{-X'<\gamma\}}J_{\text{\ab}}^{x,\max(\alpha, -X'), \gamma}(h, c-1) + \\I_{\text{\ab}}^{X',-\beta, -\gamma}(h-1, b)+\mathds{1}_{\{-X'<\beta\}}J_{\text{\ab}}^{x,\max(\gamma, -X'), \beta}(h, c-1)
\end{split}
\end{align*}
and show that $A \leq B$.

\paragraph{Case 1} If $\beta < -X'$, then $\mathds{1}_{\{-X'<\beta\}} =\mathds{1}_{\{-X'<\gamma\}} = 0$, the property then holds using the induction hypothesis for $(h-1,b)$.

\paragraph{Case 2} Now, if $\gamma \leq -X' < \beta$, $\mathds{1}_{\{-X'<\beta\}}=1$ and $\mathds{1}_{\{-X'<\gamma\}} = 0$. Moreover, $\max(\gamma,-X')=-X'$ and $\max(\alpha,-X')=-X'$. So:
$$A = I_{\text{\ab}}^{X',-\beta, -\alpha}(h-1, b)+J_{\text{\ab}}^{x,-X', \beta}(h, c-1)$$
and\begin{align*}
    \begin{split}
        B=I_{\text{\ab}}^{X',-\gamma, -\alpha}(h-1, b) + I_{\text{\ab}}^{X',-\beta, -\gamma}(h-1, b)+J_{\text{\ab}}^{x,-X', \beta}(h, c-1).
    \end{split}
\end{align*}
The inequality also directly holds using the induction hypothesis for $(h-1,b)$.

\paragraph{Case 3} Now, if $\alpha \leq -X' < \gamma$,  $\mathds{1}_{\{-X'<\beta\}}=1$ and $\mathds{1}_{\{-X'<\gamma\}} = 1$. Moreover, $\max(\gamma,-X')=\gamma$ and $\max(\alpha,-X')=-X'$. So:
$$A = I_{\text{\ab}}^{X',-\beta, -\alpha}(h-1, b)+J_{\text{\ab}}^{x,-X', \beta}(h, c-1)$$
and\begin{align*}
    \begin{split}
        B=I_{\text{\ab}}^{X',-\gamma, -\alpha}(h-1, b) + J_{\text{\ab}}^{x,-X', \beta}(h, c-1)+I_{\text{\ab}}^{X',-\beta, -\gamma}(h-1, b)+J_{\text{\ab}}^{x,\gamma, \beta}(h, c-1)
    \end{split}.
\end{align*}
Here again, the inequality holds by induction hypothesis on $(h,b)$ and using the fact $J_{\text{\ab}}^{x,\gamma, \beta}(h, c-1)$ is non-negative.

\paragraph{Case 4} Finally, if $-X' < \alpha$, we obtain:
$$A = I_{\text{\ab}}^{X',-\beta, -\alpha}(h-1, b)+J_{\text{\ab}}^{x,\alpha, \beta}(h, c-1)$$
and
\begin{align*}
    \begin{split}
        B=I_{\text{\ab}}^{X',-\gamma, -\alpha}(h-1, b) + J_{\text{\ab}}^{x,\alpha, \beta}(h, c-1)+I_{\text{\ab}}^{X',-\beta, -\gamma}(h-1, b)+J_{\text{\ab}}^{x,\gamma, \beta}(h, c-1).
    \end{split}
\end{align*}

Here the inequality holds using both the induction steps at $(h-1, b)$ and $(h, c-1)$. We have covered all possible values of $X'$, so the proof is concluded.
\end{proof}

Now we recall the theorem of interest:
\begin{Td}
    \thmtextd
\end{Td}
\begin{proof}
    The first part of the theorem is obtained by iteratively applying Proposition~\ref{prop:alphabeta} with $\alpha=-n$ and $\beta=n$ and choosing $\gamma=-n+1$, then $\gamma = -n+2$ and so on, until $\gamma=n-1$.\\
    The second part can be deducted by remarking that we have: 
    \begin{align*}
               \max_{s}I_{\text{\test}}^{x,s}(h)\leq I_{\text{\ab}}^{x,-n, n}(h)& \leq \sum_{s} I_{\text{\test}}^{x,s}(h)\leq 2n\max_{s}I_{\text{\test}}^{x,s}(h) .
    \end{align*}
    The left-hand-side inequality reflects the fact that a smaller $\alpha$$-$$\beta$ window results in evaluating strictly less nodes --- note that this doesn't hold for \scout, due to non-monotonicity of \test's complexity with respect to the threshold value $s$. By taking power $1/h$ on both sides and taking the limit in $+\infty$ this gives us the desired result.
\end{proof}

\subsection{Proof of Theorem~\ref{th:scout}}
\label{app:scout_proof}
We first recall the result:
\begin{Tc}
    \thmtextc
\end{Tc}
\begin{proof}
The proof is very similar to that of Proposition~\ref{prop:alphabeta}, in particular it relies on the same type of induction, so we only detail here the induction step for the function $J_\text{\scout}$.

We assume that
$$ I_{\text{\scout}}^{x,-n, n}(h) \leq \sum_{s=-n+1}^{n} I_{\text{\test}}^{x,s}(h)$$
and
$$ J_{\text{\scout}}^{x,-n, n}(h) \leq \sum_{s=-n+1}^{n} J_{\text{\test}}^{x,s}(h).$$

We recall the expression for $J_{\text{\scout}}$:
\begin{equation}
J_{\text{\scout}(h,c)}^{x,\alpha,\beta}=\mathbb E_{X'\geq-x}\Bigg[{I_{\text{\test}(h-1, b)}^{X', -\alpha}}+\mathds{1}_{\{\alpha<-X'\}}I^{X', -\beta,-\alpha-1}_{\text{\scout}(h-1, b)}+\mathds{1}_{\{-X'<\beta\}}J^{x, \max(\alpha, -X'),\beta}_{\text{\scout}(h,c-1)}\Bigg].
\end{equation}
A little subtlety that was not explicit in Equations~\ref{eq:i_scout}~and~\ref{eq:j_scout} is that for the special case $\alpha = \beta$, Algorithm~\ref{alg:scout} terminates instantly, incurring a cost of $0$. As a consequence when called with $\alpha =s-1$ and $\beta = s$, \scout is equivalent to a call to \textsc{test($s$)}, since $I^{X', -\beta,-\alpha-1}_{\text{\scout}(h-1, b)}=I^{X', -s,-s}_{\text{\scout}(h-1, b)}=0$.
Without loss of generality, we consider in the following $\alpha < \beta -1$, as the desired inequality is clearly true for $\alpha=\beta-1$.
Let's show the inequality holds term by term for every value of $X'$. We define:

$$A ={I_{\text{\test}(h-1, b)}^{X', -\alpha}}+\mathds{1}_{\{\alpha<-X'\}}I^{X', -\beta,-\alpha-1}_{\text{\scout}(h-1, b)}+\mathds{1}_{\{-X'<\beta\}}J^{x, \max(\alpha, -X'),\beta}_{\text{\scout}(h,c-1)} $$

and
\begin{align*}
\begin{split}
    B= \sum_{s=\alpha+1}^\beta{I_{\text{\test}(h-1, b)}^{X', -s+1}} + \sum_{s=\alpha+1}^\beta\mathds{1}_{\{-X'<s\}} {J_{\text{\test}(h, c-1)}^{x, s}}.
\end{split}
\end{align*}

Let's show that $A\leq B$ in all cases.
\paragraph{Case 1} If $\beta \leq-X'$, then $\mathds{1}_{\{\alpha<-X'\}}=1$ and $\mathds{1}_{\{-X'<\beta\}}=0$. So:
$$A ={I_{\text{\test}(h-1, b)}^{X', -\alpha}}+I^{X', -\beta,-\alpha-1}_{\text{\scout}(h-1, b)}$$
and

$$
    B= \sum_{s=\alpha+1}^\beta{I_{\text{\test}(h-1, b)}^{X', -s+1}} + \sum_{s=\alpha+1}^\beta\mathds{1}_{\{-X'<s\}} {J_{\text{\test}(h, c-1)}^{x, s}} \geq \sum_{s=\alpha+1}^\beta{I_{\text{\test}(h-1, b)}^{X', -s+1}}.
$$

By induction hypothesis on $(h-1,b)$ and variable change $s'=-s+1$, we can write:
$$A \leq {I_{\text{\test}(h-1, b)}^{X', -\alpha}}+\sum_{s=-\beta+1}^{-\alpha-1}{I_{\text{\test}(h-1, b)}^{X', s}} \leq {I_{\text{\test}(h-1, b)}^{X', -\alpha}}+\sum_{s=\alpha+2}^{\beta}{I_{\text{\test}(h-1, b)}^{X', -s+1}}\leq\sum_{s=\alpha+1}^\beta{I_{\text{\test}(h-1, b)}^{X', -s+1}} \leq B.$$
So the inequality holds in this case.

\paragraph{Case 2} If $\alpha <-X' <\beta$, then $\mathds{1}_{\{\alpha<-X'\}}=1$ and $\mathds{1}_{\{-X'<\beta\}}=1$ and $\max(\alpha,-X')=-X'$. $A$ and $B$ become:
$$A ={I_{\text{\test}(h-1, b)}^{X', -\alpha}}+I^{X', -\beta,-\alpha-1}_{\text{\scout}(h-1, b)}+J^{x, -X',\beta}_{\text{\scout}(h,c-1)} $$
and $$
    B= \sum_{s=\alpha+1}^\beta{I_{\text{\test}(h-1, b)}^{X', -s+1}} + \sum_{s=\alpha+1}^\beta\mathds{1}_{\{-X'<s\}} {J_{\text{\test}(h, c-1)}^{x, s}}.
$$\
For the term in $(h-1,b)$ it's the same as in Case 1. For the terms in $(h,c-1)$, we remark that
$$\sum_{s=\alpha+1}^\beta\mathds{1}_{\{-X'<s\}} {J_{\text{\test}(h, c-1)}^{x, s}} = \sum_{s=-X'+1}^\beta{J_{\text{\test}(h, c-1)}^{x, s}} $$
By induction hypothesis on $(h,c-1)$, with $\alpha=-X'$, we have:
$$J^{x, -X',\beta}_{\text{\scout}(h,c-1)} \leq \sum_{s=-X'+1}^\beta{J_{\text{\test}(h, c-1)}^{x, s}} .$$

So the inequality holds term by term for this case.

\paragraph{Case 3} If $-X' \leq \alpha$, we have then:
$$A ={I_{\text{\test}(h-1, b)}^{X', -\alpha}}+J^{x, -X',\beta}_{\text{\scout}(h,c-1)} $$
and $$
    B= \sum_{s=\alpha+1}^\beta{I_{\text{\test}(h-1, b)}^{X', -s+1}} + \sum_{s=\alpha+1}^\beta\mathds{1}_{\{-X'<s\}} {J_{\text{\test}(h, c-1)}^{x, s}}.
$$
By reusing arguments from the two previous cases, we can see easily that this case holds as well, thus concluding the proof.
\end{proof}

\newpage
\section{Monte-Carlo Simulations}

\label{app:montecarlo}
In this section we provide results of Monte-Carlo simulations, experimentally validating the equations characterizing the different algorithms in the paper. All experiments here, and in the main text, were run in a couple of hours of CPU time on a consumer-grade laptop.

In Figure~\ref{fig:mc_test},~\ref{fig:mc_ab} and \ref{fig:mc_scout}, we represent the evolution of the Monte-Carlo mean estimator of the \test, \scout and \alphabeta complexities, respectively. The Monte-Carlo estimator is represented as a function of the number of trials, for different settings of distribution $\mu$, branching degree $b$, value range $n$ and tree height $h$. In every scenario, the Monte-Carlo estimator converges to the oracle computed using equations derived in Section~\ref{sec:alphabeta_complexit}.  The settings were chosen to showcase a high diversity of parameters, while maintaining the computational cost reasonable. Results are averaged over 5 independent random seeds. Shaded areas represent bootstrapped 95\% confidence interval.

\begin{figure}[h]
\centering
\includegraphics[width=\textwidth]{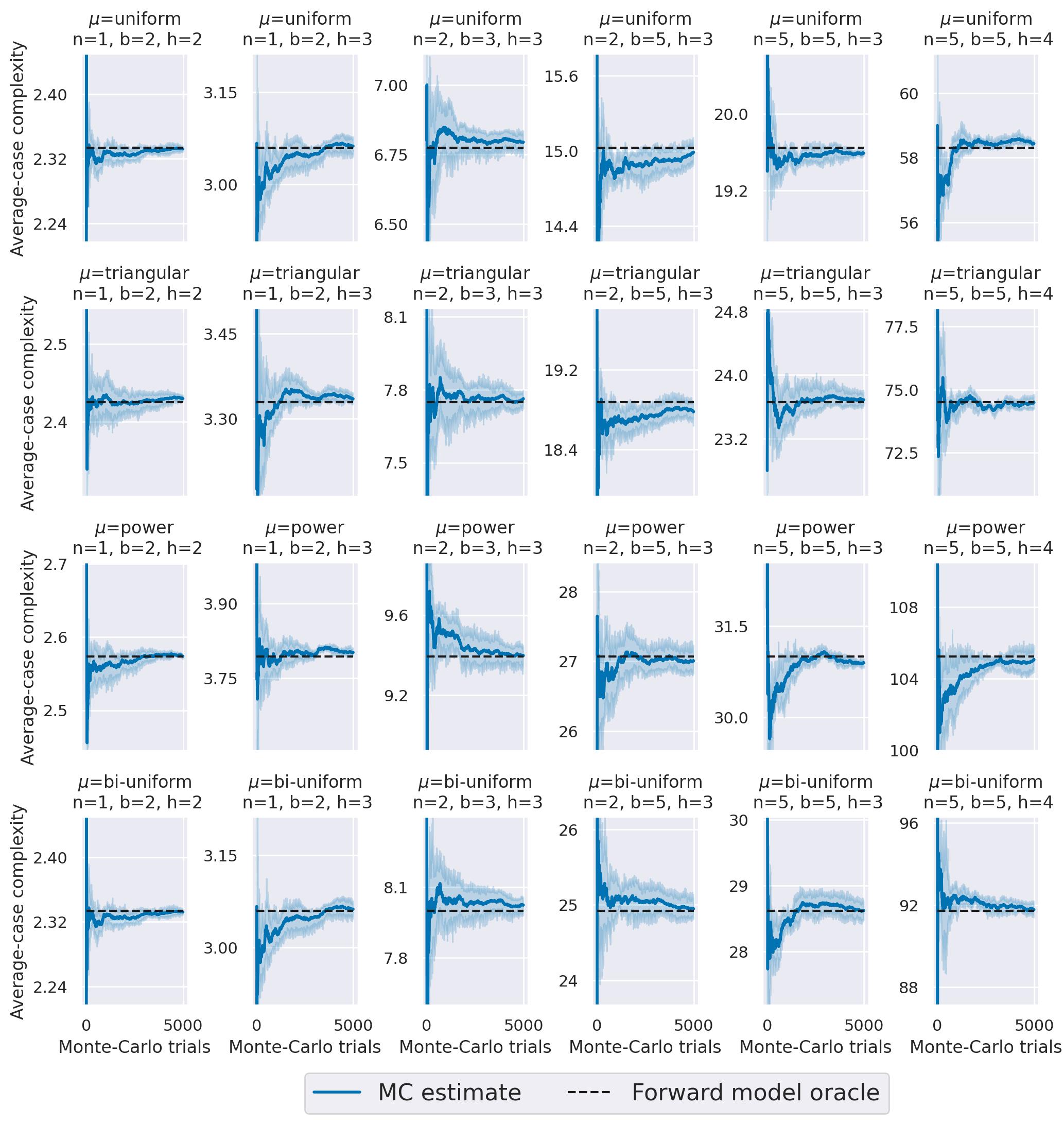}

\caption{Evolution of the Monte-Carlo mean estimator of the \test complexity, as a function of the number of trials, for different settings. Results are averaged over 5 independent random seeds and shaded areas represent bootstrapped 95\% confidence interval. The oracle is computed using Equations~\ref{eq:i_test_recurrence}~and~\ref{eq:j_test_recurrence}.}
\label{fig:mc_test}
\end{figure}

\begin{figure}[t]
\centering
\includegraphics[width=\textwidth]{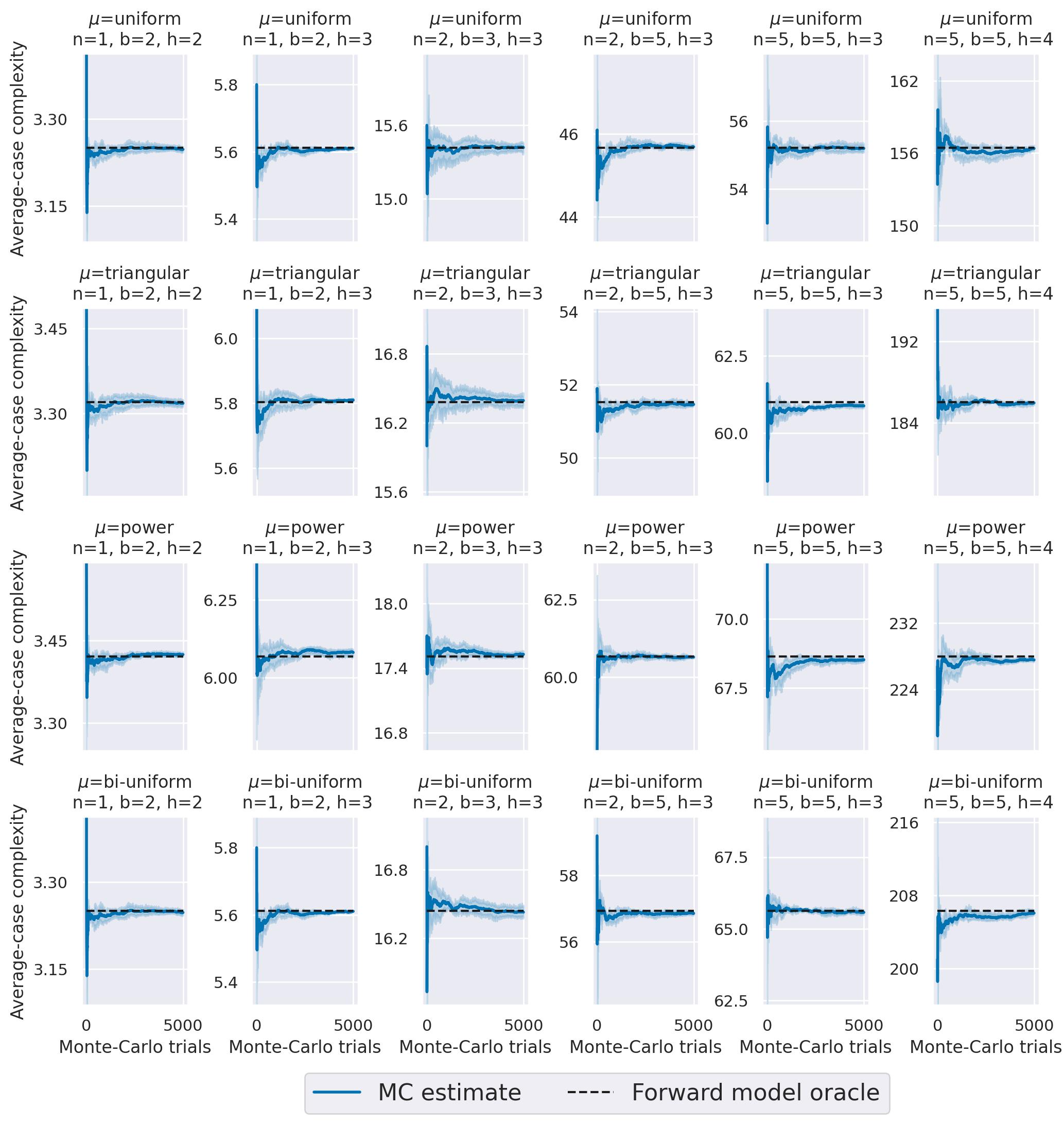}
\caption{Evolution of the Monte-Carlo mean estimator of the \alphabeta complexity, as a function of the number of trials, for different settings. Results are averaged over 5 independent random seeds and shaded areas represent bootstrapped 95\% confidence interval. The oracle is computed using Equations~\ref{equation:ab_complexity}~and~\ref{equation:ab_complexity_j}.}
\label{fig:mc_ab}
\end{figure}

\begin{figure}[t]
\centering
\includegraphics[width=\textwidth]{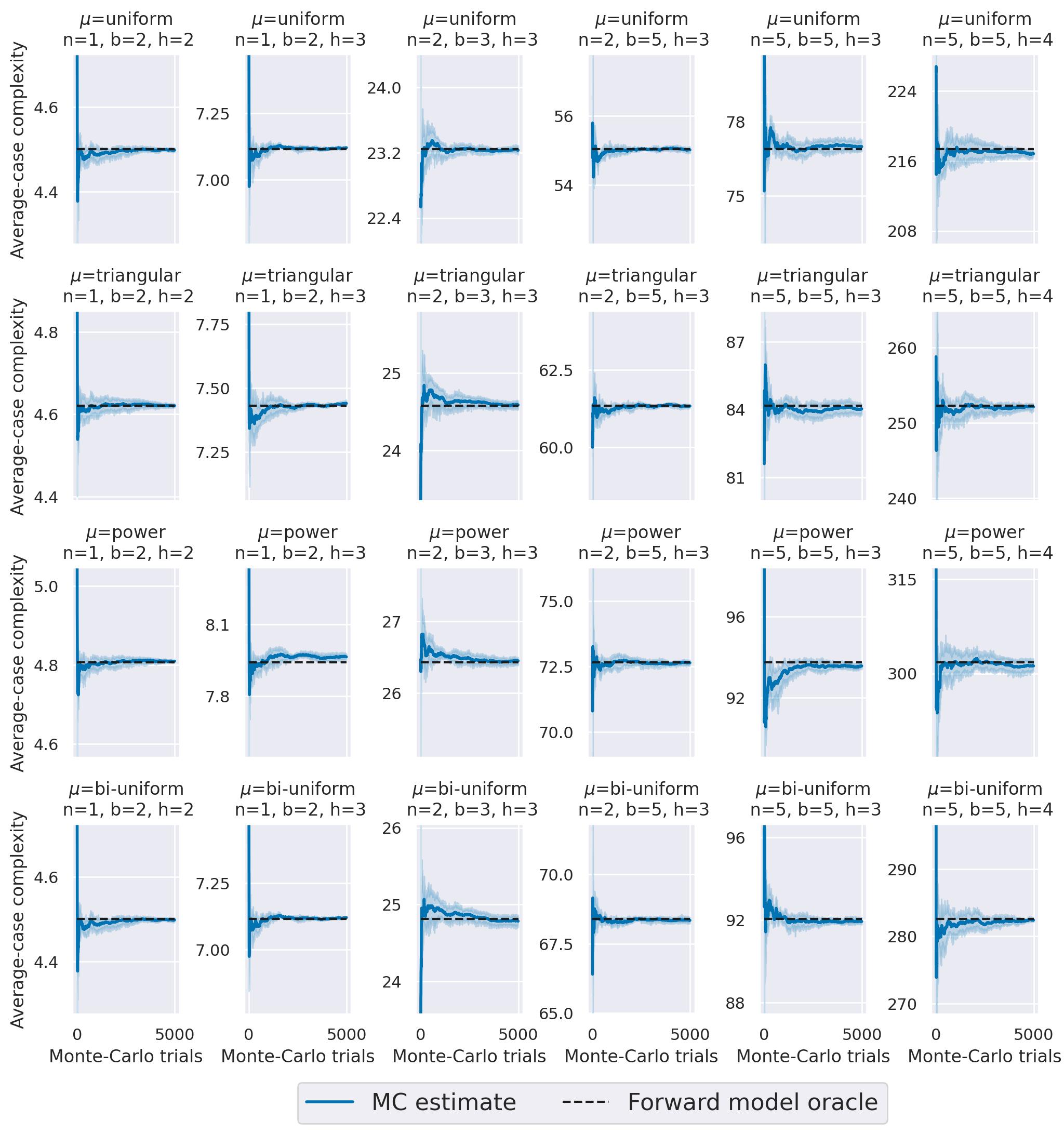}
\caption{Evolution of the Monte-Carlo mean estimator of the \scout complexity, as a function of the number of trials, for different settings. Results are averaged over 5 independent random seeds and shaded areas represent bootstrapped 95\% confidence interval. The oracle is computed using Equations~\ref{eq:i_scout}~and~\ref{eq:j_scout}.}
\label{fig:mc_scout}
\end{figure}

\clearpage
\section{Derivation details for the SOLVE analysis}
\label{app:solve}

\begin{algorithm}
\SetAlgoLined
\DontPrintSemicolon
\KwIn{Current node \textit{N}, search depth $h$}
\KwOut{Value of node \textit{N}.}
\lIf{$h = 0$}{\Return{N.value} }
$best \leftarrow 0$\;
\ForEach{N' in N.children}{
    $value \leftarrow 1-\text{SOLVE}(N',h-1)$
    
    $best \leftarrow \max(best, value )$ \\
\lIf{$best = 1$ }{\textbf{break}}}
\Return{$best$}\;
\caption{SOLVE($N, h$) --- binary-valued tree}\label{alg:solve}
\end{algorithm}

In the following, we provide additional details for the analysis of the \solve algorithm. 

We analyze the average-case complexity $I_{\text{\solve}}(h)$ of \solve on a depth-$h$ tree generated by  the \textit{forward model} where $\mu = \mathcal{B}(q)$ (Bernoulli distribution with $q$ the probability of drawing a $0$). Clearly, $I_{\text{\solve}}(h)= \mathbb{E}_{X \sim \mu}I_{\text{\solve}}^X(h)$ where $I_{\text{\solve}}^x(h)$ denotes the complexity of \solve, but conditioned on the evaluated node value $x \in \{0,1\}$. By capturing in equations the execution flow described in Algorithm~\ref{alg:solve} for every encountered case, we can characterize the dynamics of \solve. If $x=0$ all children values $x_i'$s are $1$. In this case \solve will recursively evaluate all $b$ children:
\begin{equation}
I_{\text{\solve}}^0(h) = b I_{\text{\solve}}^1(h - 1).
\label{eq:solve0}
\end{equation}
Now, if $x=1$, at least one child $x'_i$ will hold the value 0, and whenever \solve finds it, it will terminate early. Hence, \solve will incur the cost of evaluating this $x'_i=0$ child plus the expected number of failed trials needed to find it, multiplied by the cost of evaluating a $x'_i=1$ child:
\begin{equation}
I_{\text{\solve}}^1(h) = I_{\text{\solve}}^0(h - 1) + t(q,b) I_{\text{\solve}}^1(h - 1),
\label{eq:solve1}
\end{equation}
where $t(q,b)$ is the expected number of trials before finding a child with value 0. Note that we can write these equations because by design, the distribution $\mu$ is independent from the height $h$ (given the knowledge of the value $x$). Under the \textit{forward model}, $t$ can be derived by compounding the individual probabilities of finding a child with value $0$ on the first $(b-1)$ trials, leading to the expression in Equation~\ref{eq:solve_r_t}:
\begin{align*}
    t(q,b) = \sum_{k=1}^{b-1} \frac{1 + (b - k - 1)q}{b} k (1 - q)^k.
\end{align*}
Equations~\ref{eq:solve0}~and~\ref{eq:solve1} together define a recursive linear system, with initial conditions $I_{\text{\solve}}^x(0) = 1$ (a tree with only one node always incurs a cost of $1$). Luckily,  a closed-form solution for $I_{\text{\solve}}(h)$ can be derived, first we define:
\begin{align}
    r_{1,2} = \frac{ t(q,b) \pm \sqrt{t(q,b)^2 + 4b} }{2},
\end{align}
allowing us to write:
\begin{align}
    I_{\text{\solve}}(h) = q\bigg(Ar_1^h+(1-A)r_2^h\bigg)+b(1-q)\bigg(Ar_1^{h-1}+(1-A)r_2^{h-1}\bigg)
\end{align}
where $A\in[0,1]$ is defined as:
\begin{align*}
    A = \frac 1 2 + \frac {1+t(q,b)/2} {\sqrt{t(q,b)^2+4b}}.
\end{align*}

Clearly the branching factor is determined by the larger of $r_1$ and $r_2$, hence the expression in Equation~\ref{eq:solve_r_t}:
\begin{align*}
        r_{\text{\solve}} = r_1=  \frac{ t(q,b) + \sqrt{t(q,b)^2 + 4b} }{2}.
\end{align*}

\section*{NeurIPS Paper Checklist}

\begin{enumerate}

\item {\bf Claims}
    \item[] Question: Do the main claims made in the abstract and introduction accurately reflect the paper's contributions and scope?
    \item[] Answer: \answerYes{} 
    \item[] Justification:  The abstract and introduction accurately describe the paper's main contributions: the critique of the standard model, the proposal of the forward model, the derivation of recursive complexity equations for specific algorithms under this model, and the key finding regarding differing finite-depth performance despite similar asymptotic branching factors. The scope (analysis of classical deterministic game solvers) is also clearly stated.
   
    \item[] Guidelines:
    \begin{itemize}
        \item The answer NA means that the abstract and introduction do not include the claims made in the paper.
        \item The abstract and/or introduction should clearly state the claims made, including the contributions made in the paper and important assumptions and limitations. A No or NA answer to this question will not be perceived well by the reviewers. 
        \item The claims made should match theoretical and experimental results, and reflect how much the results can be expected to generalize to other settings. 
        \item It is fine to include aspirational goals as motivation as long as it is clear that these goals are not attained by the paper. 
    \end{itemize}

\item {\bf Limitations}
    \item[] Question: Does the paper discuss the limitations of the work performed by the authors?
    \item[] Answer: \answerYes{} 
    \item[] Justification: The discussion section explicitly discusses limitations, including the lack of closed-form expressions for branching factors for all algorithms, the focus on specific classical algorithms (AlphaBeta, Scout, Test) while leaving others like MTD(f) and PVS for future work, and the current focus on discrete-valued trees.
    \item[] Guidelines:
    \begin{itemize}
        \item The answer NA means that the paper has no limitation while the answer No means that the paper has limitations, but those are not discussed in the paper. 
        \item The authors are encouraged to create a separate "Limitations" section in their paper.
        \item The paper should point out any strong assumptions and how robust the results are to violations of these assumptions (e.g., independence assumptions, noiseless settings, model well-specification, asymptotic approximations only holding locally). The authors should reflect on how these assumptions might be violated in practice and what the implications would be.
        \item The authors should reflect on the scope of the claims made, e.g., if the approach was only tested on a few datasets or with a few runs. In general, empirical results often depend on implicit assumptions, which should be articulated.
        \item The authors should reflect on the factors that influence the performance of the approach. For example, a facial recognition algorithm may perform poorly when image resolution is low or images are taken in low lighting. Or a speech-to-text system might not be used reliably to provide closed captions for online lectures because it fails to handle technical jargon.
        \item The authors should discuss the computational efficiency of the proposed algorithms and how they scale with dataset size.
        \item If applicable, the authors should discuss possible limitations of their approach to address problems of privacy and fairness.
        \item While the authors might fear that complete honesty about limitations might be used by reviewers as grounds for rejection, a worse outcome might be that reviewers discover limitations that aren't acknowledged in the paper. The authors should use their best judgment and recognize that individual actions in favor of transparency play an important role in developing norms that preserve the integrity of the community. Reviewers will be specifically instructed to not penalize honesty concerning limitations.
    \end{itemize}

\item {\bf Theory assumptions and proofs}
    \item[] Question: For each theoretical result, does the paper provide the full set of assumptions and a complete (and correct) proof?
    \item[] Answer: \answerYes{}
    \item[] Justification: The paper states its main theoretical results  and indicates that proofs are provided in the appendix. The core derivations of complexity recurrences proceed from the definition of the forward model and the algorithms. Key assumptions about the model (uniform special child selection, fixed distribution $\mu$) are stated. Theorems and formulas are numbered and referenced.
    \item[] Guidelines:
    \begin{itemize}
        \item The answer NA means that the paper does not include theoretical results. 
        \item All the theorems, formulas, and proofs in the paper should be numbered and cross-referenced.
        \item All assumptions should be clearly stated or referenced in the statement of any theorems.
        \item The proofs can either appear in the main paper or the supplemental material, but if they appear in the supplemental material, the authors are encouraged to provide a short proof sketch to provide intuition. 
        \item Inversely, any informal proof provided in the core of the paper should be complemented by formal proofs provided in appendix or supplemental material.
        \item Theorems and Lemmas that the proof relies upon should be properly referenced. 
    \end{itemize}

    \item {\bf Experimental result reproducibility}
    \item[] Question: Does the paper fully disclose all the information needed to reproduce the main experimental results of the paper to the extent that it affects the main claims and/or conclusions of the paper (regardless of whether the code and data are provided or not)?
    \item[] Answer: \answerYes{} 
    \item[] Justification: The main "experimental" results are numerical computations based on the derived theoretical recurrence relations. The paper provides these equations, describes the method for solving them (matrix iteration, spectral radius), specifies the parameters used, and describes/plots the distributions $\mu$ used in the figures. Additionally, code is included in the supplementary material.
  
    \item[] Guidelines:
    \begin{itemize}
        \item The answer NA means that the paper does not include experiments.
        \item If the paper includes experiments, a No answer to this question will not be perceived well by the reviewers: Making the paper reproducible is important, regardless of whether the code and data are provided or not.
        \item If the contribution is a dataset and/or model, the authors should describe the steps taken to make their results reproducible or verifiable. 
        \item Depending on the contribution, reproducibility can be accomplished in various ways. For example, if the contribution is a novel architecture, describing the architecture fully might suffice, or if the contribution is a specific model and empirical evaluation, it may be necessary to either make it possible for others to replicate the model with the same dataset, or provide access to the model. In general. releasing code and data is often one good way to accomplish this, but reproducibility can also be provided via detailed instructions for how to replicate the results, access to a hosted model (e.g., in the case of a large language model), releasing of a model checkpoint, or other means that are appropriate to the research performed.
        \item While NeurIPS does not require releasing code, the conference does require all submissions to provide some reasonable avenue for reproducibility, which may depend on the nature of the contribution. For example
        \begin{enumerate}
            \item If the contribution is primarily a new algorithm, the paper should make it clear how to reproduce that algorithm.
            \item If the contribution is primarily a new model architecture, the paper should describe the architecture clearly and fully.
            \item If the contribution is a new model (e.g., a large language model), then there should either be a way to access this model for reproducing the results or a way to reproduce the model (e.g., with an open-source dataset or instructions for how to construct the dataset).
            \item We recognize that reproducibility may be tricky in some cases, in which case authors are welcome to describe the particular way they provide for reproducibility. In the case of closed-source models, it may be that access to the model is limited in some way (e.g., to registered users), but it should be possible for other researchers to have some path to reproducing or verifying the results.
        \end{enumerate}
    \end{itemize}

\item {\bf Open access to data and code}
    \item[] Question: Does the paper provide open access to the data and code, with sufficient instructions to faithfully reproduce the main experimental results, as described in supplemental material?
    \item[] Answer: \answerYes{} 
    \item[] Justification: Python code included in the supplementary material. No external data is used.

    \item[] Guidelines:
    \begin{itemize}
        \item The answer NA means that paper does not include experiments requiring code.
        \item Please see the NeurIPS code and data submission guidelines (\url{https://nips.cc/public/guides/CodeSubmissionPolicy}) for more details.
        \item While we encourage the release of code and data, we understand that this might not be possible, so “No” is an acceptable answer. Papers cannot be rejected simply for not including code, unless this is central to the contribution (e.g., for a new open-source benchmark).
        \item The instructions should contain the exact command and environment needed to run to reproduce the results. See the NeurIPS code and data submission guidelines (\url{https://nips.cc/public/guides/CodeSubmissionPolicy}) for more details.
        \item The authors should provide instructions on data access and preparation, including how to access the raw data, preprocessed data, intermediate data, and generated data, etc.
        \item The authors should provide scripts to reproduce all experimental results for the new proposed method and baselines. If only a subset of experiments are reproducible, they should state which ones are omitted from the script and why.
        \item At submission time, to preserve anonymity, the authors should release anonymized versions (if applicable).
        \item Providing as much information as possible in supplemental material (appended to the paper) is recommended, but including URLs to data and code is permitted.
    \end{itemize}

\item {\bf Experimental setting/details}
    \item[] Question: Does the paper specify all the training and test details (e.g., data splits, hyperparameters, how they were chosen, type of optimizer, etc.) necessary to understand the results?
    \item[] Answer: \answerYes{} 
    \item[] Justification: All parameters and settings are described.
    \item[] Guidelines:
    \begin{itemize}
        \item The answer NA means that the paper does not include experiments.
        \item The experimental setting should be presented in the core of the paper to a level of detail that is necessary to appreciate the results and make sense of them.
        \item The full details can be provided either with the code, in appendix, or as supplemental material.
    \end{itemize}

\item {\bf Experiment statistical significance}
    \item[] Question: Does the paper report error bars suitably and correctly defined or other appropriate information about the statistical significance of the experiments?
    \item[] Answer: \answerYes{} 
    \item[] Justification: The main results in Figure~\ref{fig:results} are exact calculations of expected values from the theoretical model, not empirical results subject to sampling variance. Therefore, statistical error bars are not applicable to these central findings. For the Monte-Carlo simulations in the Appendix~\ref{app:montecarlo}, number of seeds and confidence intervals are included.
   
    \item[] Guidelines:
    \begin{itemize}
        \item The answer NA means that the paper does not include experiments.
        \item The authors should answer "Yes" if the results are accompanied by error bars, confidence intervals, or statistical significance tests, at least for the experiments that support the main claims of the paper.
        \item The factors of variability that the error bars are capturing should be clearly stated (for example, train/test split, initialization, random drawing of some parameter, or overall run with given experimental conditions).
        \item The method for calculating the error bars should be explained (closed form formula, call to a library function, bootstrap, etc.)
        \item The assumptions made should be given (e.g., Normally distributed errors).
        \item It should be clear whether the error bar is the standard deviation or the standard error of the mean.
        \item It is OK to report 1-sigma error bars, but one should state it. The authors should preferably report a 2-sigma error bar than state that they have a 96\% CI, if the hypothesis of Normality of errors is not verified.
        \item For asymmetric distributions, the authors should be careful not to show in tables or figures symmetric error bars that would yield results that are out of range (e.g., negative error rates).
        \item If error bars are reported in tables or plots, The authors should explain in the text how they were calculated and reference the corresponding figures or tables in the text.
    \end{itemize}

\item {\bf Experiments compute resources}
    \item[] Question: For each experiment, does the paper provide sufficient information on the computer resources (type of compute workers, memory, time of execution) needed to reproduce the experiments?
    \item[] Answer: \answerYes{} 
    \item[] Justification: Compute resources are provided in Appendix~\ref{app:montecarlo}.
    \item[] Guidelines:
    \begin{itemize}
        \item The answer NA means that the paper does not include experiments.
        \item The paper should indicate the type of compute workers CPU or GPU, internal cluster, or cloud provider, including relevant memory and storage.
        \item The paper should provide the amount of compute required for each of the individual experimental runs as well as estimate the total compute. 
        \item The paper should disclose whether the full research project required more compute than the experiments reported in the paper (e.g., preliminary or failed experiments that didn't make it into the paper). 
    \end{itemize}
    
\item {\bf Code of ethics}
    \item[] Question: Does the research conducted in the paper conform, in every respect, with the NeurIPS Code of Ethics \url{https://neurips.cc/public/EthicsGuidelines}?
    \item[] Answer: \answerYes{}{} 
    \item[] Justification: The research involves theoretical modeling and algorithmic analysis. It does not involve human subjects, sensitive data, or direct applications with obvious ethical concerns related to fairness, privacy, or potential misuse as outlined in the Code of Ethics.
   
    \item[] Guidelines:
    \begin{itemize}
        \item The answer NA means that the authors have not reviewed the NeurIPS Code of Ethics.
        \item If the authors answer No, they should explain the special circumstances that require a deviation from the Code of Ethics.
        \item The authors should make sure to preserve anonymity (e.g., if there is a special consideration due to laws or regulations in their jurisdiction).
    \end{itemize}

\item {\bf Broader impacts}
    \item[] Question: Does the paper discuss both potential positive societal impacts and negative societal impacts of the work performed?
    \item[] Answer: \answerNA{}
    \item[] Justification: The paper focuses on foundational theoretical analysis of algorithms and does not include a discussion of potential broader societal impacts. The work is not directly tied to specific applications where such impacts would be immediate.
  
    \item[] Guidelines:
    \begin{itemize}
        \item The answer NA means that there is no societal impact of the work performed.
        \item If the authors answer NA or No, they should explain why their work has no societal impact or why the paper does not address societal impact.
        \item Examples of negative societal impacts include potential malicious or unintended uses (e.g., disinformation, generating fake profiles, surveillance), fairness considerations (e.g., deployment of technologies that could make decisions that unfairly impact specific groups), privacy considerations, and security considerations.
        \item The conference expects that many papers will be foundational research and not tied to particular applications, let alone deployments. However, if there is a direct path to any negative applications, the authors should point it out. For example, it is legitimate to point out that an improvement in the quality of generative models could be used to generate deepfakes for disinformation. On the other hand, it is not needed to point out that a generic algorithm for optimizing neural networks could enable people to train models that generate Deepfakes faster.
        \item The authors should consider possible harms that could arise when the technology is being used as intended and functioning correctly, harms that could arise when the technology is being used as intended but gives incorrect results, and harms following from (intentional or unintentional) misuse of the technology.
        \item If there are negative societal impacts, the authors could also discuss possible mitigation strategies (e.g., gated release of models, providing defenses in addition to attacks, mechanisms for monitoring misuse, mechanisms to monitor how a system learns from feedback over time, improving the efficiency and accessibility of ML).
    \end{itemize}
    
\item {\bf Safeguards}
    \item[] Question: Does the paper describe safeguards that have been put in place for responsible release of data or models that have a high risk for misuse (e.g., pretrained language models, image generators, or scraped datasets)?
    \item[] Answer: \answerNA{} 
    \item[] Justification: The paper does not release data or models that pose a high risk for misuse. The released asset is code for theoretical analysis.
 
    \item[] Guidelines:
    \begin{itemize}
        \item The answer NA means that the paper poses no such risks.
        \item Released models that have a high risk for misuse or dual-use should be released with necessary safeguards to allow for controlled use of the model, for example by requiring that users adhere to usage guidelines or restrictions to access the model or implementing safety filters. 
        \item Datasets that have been scraped from the Internet could pose safety risks. The authors should describe how they avoided releasing unsafe images.
        \item We recognize that providing effective safeguards is challenging, and many papers do not require this, but we encourage authors to take this into account and make a best faith effort.
    \end{itemize}

\item {\bf Licenses for existing assets}
    \item[] Question: Are the creators or original owners of assets (e.g., code, data, models), used in the paper, properly credited and are the license and terms of use explicitly mentioned and properly respected?
    \item[] Answer: \answerNA{} 
    \item[] Justification: The paper does not use existing assets.
    \item[] Guidelines:
    \begin{itemize}
        \item The answer NA means that the paper does not use existing assets.
        \item The authors should cite the original paper that produced the code package or dataset.
        \item The authors should state which version of the asset is used and, if possible, include a URL.
        \item The name of the license (e.g., CC-BY 4.0) should be included for each asset.
        \item For scraped data from a particular source (e.g., website), the copyright and terms of service of that source should be provided.
        \item If assets are released, the license, copyright information, and terms of use in the package should be provided. For popular datasets, \url{paperswithcode.com/datasets} has curated licenses for some datasets. Their licensing guide can help determine the license of a dataset.
        \item For existing datasets that are re-packaged, both the original license and the license of the derived asset (if it has changed) should be provided.
        \item If this information is not available online, the authors are encouraged to reach out to the asset's creators.
    \end{itemize}

\item {\bf New assets}
    \item[] Question: Are new assets introduced in the paper well documented and is the documentation provided alongside the assets?
    \item[] Answer: \answerYes{} 
    \item[] Justification: The paper introduces new code for complexity calculations, included in supplementary material and released upon acceptance. The code is concise and documented.
  
    \item[] Guidelines:
    \begin{itemize}
        \item The answer NA means that the paper does not release new assets.
        \item Researchers should communicate the details of the dataset/code/model as part of their submissions via structured templates. This includes details about training, license, limitations, etc. 
        \item The paper should discuss whether and how consent was obtained from people whose asset is used.
        \item At submission time, remember to anonymize your assets (if applicable). You can either create an anonymized URL or include an anonymized zip file.
    \end{itemize}

\item {\bf Crowdsourcing and research with human subjects}
    \item[] Question: For crowdsourcing experiments and research with human subjects, does the paper include the full text of instructions given to participants and screenshots, if applicable, as well as details about compensation (if any)? 
    \item[] Answer: \answerNA{} 
    \item[] Justification: The research does not involve crowdsourcing or human subjects.
   
    \item[] Guidelines:
    \begin{itemize}
        \item The answer NA means that the paper does not involve crowdsourcing nor research with human subjects.
        \item Including this information in the supplemental material is fine, but if the main contribution of the paper involves human subjects, then as much detail as possible should be included in the main paper. 
        \item According to the NeurIPS Code of Ethics, workers involved in data collection, curation, or other labor should be paid at least the minimum wage in the country of the data collector. 
    \end{itemize}

\item {\bf Institutional review board (IRB) approvals or equivalent for research with human subjects}
    \item[] Question: Does the paper describe potential risks incurred by study participants, whether such risks were disclosed to the subjects, and whether Institutional Review Board (IRB) approvals (or an equivalent approval/review based on the requirements of your country or institution) were obtained?
    \item[] Answer: \answerNA{} 
    \item[] Justification: The research does not involve human subjects.
   
    \item[] Guidelines:
    \begin{itemize}
        \item The answer NA means that the paper does not involve crowdsourcing nor research with human subjects.
        \item Depending on the country in which research is conducted, IRB approval (or equivalent) may be required for any human subjects research. If you obtained IRB approval, you should clearly state this in the paper. 
        \item We recognize that the procedures for this may vary significantly between institutions and locations, and we expect authors to adhere to the NeurIPS Code of Ethics and the guidelines for their institution. 
        \item For initial submissions, do not include any information that would break anonymity (if applicable), such as the institution conducting the review.
    \end{itemize}

\item {\bf Declaration of LLM usage}
    \item[] Question: Does the paper describe the usage of LLMs if it is an important, original, or non-standard component of the core methods in this research? Note that if the LLM is used only for writing, editing, or formatting purposes and does not impact the core methodology, scientific rigorousness, or originality of the research, declaration is not required.
    \item[] Answer:\answerNA{}
    \item[] Justification: LLMs were only used for proofreading and editing purposes.
    \item[] Guidelines:
    \begin{itemize}
        \item The answer NA means that the core method development in this research does not involve LLMs as any important, original, or non-standard components.
        \item Please refer to our LLM policy (\url{https://neurips.cc/Conferences/2025/LLM}) for what should or should not be described.
    \end{itemize}

\end{enumerate}

\end{document}